\documentclass{article}

\usepackage[final]{neurips_2022}

\usepackage[utf8]{inputenc} 
\usepackage[T1]{fontenc}    
\usepackage[hidelinks]{hyperref}       
\hypersetup{
    colorlinks,
    linkcolor={red!50!black},
    citecolor={blue!50!black},
    urlcolor={blue!80!black}
}
\usepackage{url}            
\usepackage{booktabs}       
\usepackage{amsfonts}       
\usepackage{nicefrac}       
\usepackage{microtype}      
\usepackage{xcolor}         
\usepackage{verbatim}
\usepackage{cjhebrew}  
\usepackage{proof}
\usepackage{amsthm}
\usepackage{amssymb}
\usepackage{natbib}  
\newlength\mylen
\setcitestyle{numbers,square,citesep={,\kern-\mylen}}
\usepackage{dirtytalk}  
\usepackage{graphicx}
\usepackage{subcaption}  
\usepackage{amsmath}
\usepackage{stix}
\usepackage{multirow}
\usepackage{tikz}
\usetikzlibrary{arrows.meta,positioning,calc}
\usepackage{dsfont}
\usepackage{comment}
\usepackage[export]{adjustbox}
\usepackage{wrapfig}
\usepackage{enumitem}

\usepackage{tikz}
\usetikzlibrary{shapes,decorations,arrows,calc,arrows.meta,fit,positioning}
\tikzset{
    -Latex,auto,node distance =1 cm and 1 cm,semithick,
    state/.style ={ellipse, draw, minimum width = 0.7 cm},
    point/.style = {circle, draw, inner sep=0.04cm,fill,node contents={}},
    bidirected/.style={Latex-Latex,dashed},
    el/.style = {inner sep=2pt, align=left, sloped}
}

\PassOptionsToPackage{numbers, compress, semicolon}{natbib}

\usepackage{inconsolata}
\usepackage{caption}

\usepackage{units}

\newcommand{\red}[1]{{\leavevmode\color{red}{#1}}}

\newcommand{\green}[1]{{\leavevmode\color[RGB]{0,128,0}{#1}}}

\newcommand{\orange}[1]{{\leavevmode\color[RGB]{255,127,39}{#1}}}

\newcommand{\darkgray}[1]{{\leavevmode\color[RGB]{140,140,140}{#1}}}
\newcommand{\gray}[1]{{\leavevmode\color[RGB]{180,180,180}{#1}}}
\newcommand\todo[1]{\red{TODO: {#1}}}
\newcommand\tocite{\red{[CITE]}}

\newcommand\expect[2]{\mathbb{E}_{#1}{\left[ {#2} \right]}}
\newcommand\prob[2]{P_{#1}{\left( {#2} \right)}}
\DeclareMathOperator*{\argmax}{argmax}

\DeclareMathOperator*{\argmin}{argmin}

\newcommand{\one}[1]{\mathds{1}{\{{#1}\}}}
\newcommand{\indep}{\bot}

\newcommand{\naive}{na\"ive}
\newcommand{\Naive}{Na\"ive}

\newcommand{\xbar}{{\bar{x}}}

\newcommand{\R}{\mathbb{R}}
\newcommand{\yhat}{{\hat{y}}}
\newcommand{\vtilde}{{\tilde{v}}}

\newcommand{\dist}{D}
\newcommand{\env}{{e}}
\newcommand{\envs}{{\cal{E}}}
\newcommand{\envstrn}{{\envs_{\mathrm{train}}}}
\newcommand{\disttrn}{{\dist_{\mathrm{train}}}}
\newcommand{\smplst}{{S}}
\newcommand{\loss}{{L}}
\newcommand{\reg}{{R}}
\newcommand{\carrow}{\textcolor{blue}{\rightarrow}}
\newcommand{\acarrow}{\textcolor{red}{\leftarrow}}
\newcommand{\belarrow}{\textcolor{teal}{\rightarrow}}
\newcommand{\skeparrow}{\textcolor{brown}{\leftarrow}}
\newcommand{\uc}{{u_{\carrow y}}} 
\newcommand{\uac}{{u_{\acarrow y}}} 
\newcommand{\ubel}{{u_{x \belarrow r}}} 
\newcommand{\uskep}{{u_{x \skeparrow r}}}

\newcommand{\fbel}{{f_{x \belarrow r}}} 
\newcommand{\fskep}{{f_{x \skeparrow r}}}

\newcommand{\MMD}{\mathrm{MMD}}

\newtheorem{proposition}{Proposition}
\newtheorem{definition}{Definition}
\newtheorem{corollary}{Corollary}
\newtheorem{lemma}{Lemma}

\title{In the Eye of the Beholder: \\ Robust Prediction with Causal User Modeling}

\author{%
  Amir Feder \\
  Columbia University \\
  \texttt{amir.feder@columbia.edu} \\ 
  \And 
  Guy Horowitz \\
  Technion \\
  \And
  Yoav Wald \\
  Johns Hopkins University \\
  \And
  Roi Reichart \\
  Technion \\
  \And
  Nir Rosenfeld \\
  Technion \\
}

\begin{document}

\maketitle

\begin{abstract}
    Accurately predicting the relevance of items to users is crucial to the success of many social platforms. Conventional approaches train models on logged historical data; but recommendation systems, media services, and online marketplaces all exhibit a constant influx of new content---making relevancy a moving target, to which standard predictive models are not robust. In this paper, we propose a learning framework for relevance prediction that is robust to changes in the data distribution. Our key observation is that robustness can be obtained by accounting for \emph{how users causally perceive the environment}. We model users as boundedly-rational decision makers whose causal beliefs are encoded by a causal graph, and show how minimal information regarding the graph can be used to contend with distributional changes. Experiments in multiple settings demonstrate the effectiveness of our approach.
\end{abstract}

\section{Introduction}
\label{sec:intro}

Across a multitude of domains and applications,
machine learning has become imperative for guiding human users in many of the decisions they make \citep{siddiqi2012credit,wuest2016machine,callahan2017machine}.
From recommendation systems and search engines to e-commerce platforms and online marketplaces,
learned models are regularly used to filter content, rank items, and display select information---all with the primary intent of helping users choose items that are relevant to them.
The predominant approach for learning in these tasks
is to train models to accurately predict the relevance of items to users.
But since training is often carried out on logged historical records,
even highly-accurate models remain calibrated to the distribution of \emph{previously} observed data on which they were trained \cite{bonner2018causal, zhang2021causal, wang2021deconfounded}.
Given that in virtually any online platform the distribution of content naturally varies over time and location---due to trends and fashions, innovation, or forces of supply and demand---models trained on logged data
may fail to correctly predict the choices and preferences of users on unseen, future distributions \citep{dandekar2013biased, abdollahpouri2017controlling, faddoul2020longitudinal, krauth2020offline, mladenov2021recsim}.

In this paper, we present a novel conceptual framework for learning predictive models
of user-item relevance that are robust to changes in the underlying data distribution.
Our approach is built around two key observations:
(i) that relevance to users is determined by the way in which users \emph{perceive} value,
and (ii) that this process of value attribution is \emph{causal} in nature.
As an example, consider a video streaming service
in which a user $u$ is trying to determine whether watching a certain movie will be worthwhile.
To make this decision, $y \in \{0,1\}$, the user
has at her disposal a feature description of the movie, $x$, and a system-generated, personalized relevance score, $r$
(e.g., ``a 92\% match!'').
How will she integrate these two informational sources into a decision?
We argue that this crucially hinges on her belief as to \emph{why}
a particular relevance score is coupled to a particular movie.
For example, if a movie boasts a high relevance score, 
then she might suppose this score was given \emph{because} the system believes the user would like this movie.
Another user, however, may reason differently,
and instead believe that high relevance scores are given \emph{because} movies are sponsored;
if she suspects this to be a likely scenario, her reasoning should have a stark 
effect on her choices.
In both cases above,
perceived values (and the actions that follow)
stem from how each user causally interprets the recommendation environment
$\env$,
and the underlying causal structure determines how belief regarding value changes, or does not, in response to changes in important variables (e.g., in $u$, $x$, or $r$).

Here, we show how 
knowledge regarding 
the causal perceptions of users
can be leveraged for providing
distributional robustness in learning.
A primary concern for robust learning is the reliance of predictions
on spurious correlations \cite{arjovsky2019invariant};
here we argue that spuriousness can \emph{result} from
causal perceptions underlying user choice behavior.
To see the relation between causal perceptions and spuriousness, 
assume that in our movies example above,
the training data exhibits a strong correlation
between users' choices of movies, $y$, 
and a `genre' feature, $x_g$.
A predictive model optimized for accuracy will likely exploit this association, and rely on $x_g$ for prediction. 
Now, further assume that what \emph{realy} drives user satisfaction is
`production quality', $x_q$;
if $x_g$ and $x_q$ are \emph{spuriously} correlated in the training data,
then once the distribution of genres naturally changes over time,
the predictive model can fail:
the association between $x_g$ and $y$,
on which predictions rely, may no longer hold.

In essence, our approach casts robust prediction of personalized relevance as
a problem of out-of-distribution (OOD) learning,
but carefully tailored to settings where data generation is governed
by users' causal beliefs and corresponding behavior.
There is a growing recognition of how
a causal understanding of the learning environment can improve cross-domain generalization \citep{arjovsky2019invariant, wald2021calibration};
our key conceptual contribution is the observation that,
in relevance prediction,
users' perceptions \emph{are} the causal environment.
Thus, there is no `true' causal graph---all is in the eye of the beholder.
To cope with this,
we model users as reasoning about decisions through a causal graph \citep{pearl2009causality, spiegler2020behavioral}---thus allowing
our approach to anticipate how changes in the data translate to changes in user behavior.
Building on this idea, 
as well as on recent advances in the use of causal modeling for out-of-distribution learning \citep{arjovsky2019invariant, wald2021calibration,veitch2021counterfactual},
we show how various levels of knowledge regarding users' causal beliefs---whether inferred or assumed---can be utilized for learning distributionally robust predictive models.

\begin{wrapfigure}{r}{0.47\textwidth}
\centering
\captionsetup[subfigure]{labelformat=empty}
    \begin{subfigure}{0.17\textwidth}
    \begin{center}
    \begin{tikzpicture}[node distance = 1.2cm]
        \node[circle, draw, text centered] (x) {$x,r$};
        \node[circle, draw, right of = x, text centered] (y) {$y$};
        \node[circle, draw, left of = x, text centered] (e) {$\env$};

        \draw[blue, ->, line width = 1,transform canvas={yshift=1.5mm}] (x) -- (y);
        \draw[red, ->, line width = 1,transform canvas={yshift=-1.5mm}] (y) -- (x);
        \draw[->, line width = 1] (e) -- (x);
        \draw[<->, line width = 1, dash dot] (e) to[out=90, in=90, distance=0.5cm] (y);
        \end{tikzpicture}
        \end{center}
    \caption{(a)}\label{graph-causal-1}
    \end{subfigure}
    \quad
    \begin{subfigure}{0.25\textwidth}
    \begin{center}
    \begin{tikzpicture}[node distance = 1.2cm]
        \node[circle, draw, text centered] (X) {$x$};
        \node[circle, draw, below of = X] (R) {$r$};
        \node[ circle, draw, right of = X, text centered] (y) {$y$};
        \node[circle, draw, left of = X, text centered] (e) {$\env$};

        \draw[->, line width = 1] (X) -- (y);
        \draw[teal, ->, line width = 1,transform canvas={xshift=1.5mm}] (X) -- (R);
        \draw[brown, ->, line width = 1,transform canvas={xshift=-1.5mm}] (R) -- (X);
        \draw[->, line width = 1] (e) -- (X);
        \draw[->, line width = 1] (R) -- (y);
        \draw[<->, line width = 1, dash dot] (e) to[out=90, in=90, distance=0.5cm] (y);
    \end{tikzpicture}
    \end{center}
    \caption{(b)}\label{graph-causal-2}
    \end{subfigure}
    \caption{
    Simplified graphs describing users of different:
    (a) \emph{classes}: \textcolor{blue}{\textit{causal}} or \textcolor{red}{\textit{anti-causal}}, and
    (b) \emph{subclasses}: \textcolor{teal}{\textit{believer}} or \textcolor{brown}{\textit{skeptic}}
    (here shown for a causal user).
    Dashed lines indicate possible spuriousness
    (e.g., via selection). 
    }
    \label{graph-causal}
    \vspace{-2mm}
\end{wrapfigure}
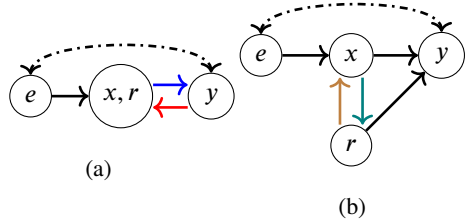

To encourage predictions $\yhat$ to be invariant to changes in the recommendation environment $\env$,
our approach enforces independence between $\yhat$ and $e$ (possibly given $y$).
This is achieved through regularization,
which penalizes predictive models for relying on $\env$ \citep{gretton2012kernel}.
In general, different graphs require different regularization schemes \citep{veitch2021counterfactual};
in our context, this would seem to imply that invariant learning requires
precise information regarding each user's causal graph.
However, our key observation is that, for user graphs,
it suffices to know which of two \emph{classes} the graph belongs to---\emph{causal} or \emph{anti-causal}--- determined by the direction of edges between $x,r$ and $y$ (Fig. \ref{graph-causal-1}).
Thus, correctly determining which regularization to apply
requires only minimal graph knowledge. 

Nonetheless, more fine-grained information can still be useful.
We show the following novel result:
if two users generate the same data, but differ in their underlying graph,
they will have different optimal \emph{out-of-distribution} predictive models
(despite sharing the same optimal in-distribution model).
The reason for this is that, to achieve robustness, regularizing for independence will
result in the discarding of different information for each user.
Operationally, this means that learning should include different models for each user-type
(not doing so implicitly constrains the models to be the same).
Here again we show that minimal additional information is useful,
and focus on subclasses of graphs that differ only in the direction of the edges between $x$ and $r$, which give rise to two user subclasses:
\textcolor{teal}{\textit{believers}} and \textcolor{brown}{\textit{skeptics}}
(Fig. \ref{graph-causal-2}).
Nonetheless, our result on differing optimal models
applies more broadly,
and may be of general interest for causal learning.

We end with a thorough empirical evaluation of our approach (\S \ref{sec:exp}),
where we explore the benefits
of different forms of knowledge regarding 
users' causal beliefs:
whether they are \textit{casual} or \textit{anti-causal},
and 
whether they are
{\textit{believers}} or {\textit{skeptic}s}.
Our results show that learning in a way that accounts for users' causal perception has significant advantages on out-of-distribution tasks.
We also study the degree to which imprecise graph knowledge is useful;
our results here show that even a rough estimate of a user's class
is sufficient for improved performance,
suggesting that our approach can be effectively applied on the basis of domain knowledge or reasonable prior beliefs.
Conversely, our results also imply that \emph{not} accounting for causal aspects of user decision-making,
or modelling them wrongly, can result in poor out-of-distribution performance. 
As systems often also play a role in determining what information is presented to users (e.g., providing  $r$),
understanding possible failure modes---and how to obtain robustness---becomes vital. 

\textbf{Broader aims. }
We aim to promote within machine learning the idea of modeling users as active and autonomous decision-makers,
with emphasis on capturing realistic aspects of human decision-making.
Our approach crucially hinges on modelling users as decision-makers that
(i) reason \emph{causally},
(ii) are \emph{boundedly-rational},  and
(iii) must cope with \emph{uncertainty}.
As we will show,
all three are key to our framework, and operate in unison.
Each of the aspects above relates to one of three main pillars on which modern theories of decision-making stand \cite{spiegler2011bounded},
thus blending three different fields---discrete choice (economics),
causal modeling (statistics),
and domain generalization (machine learning).

\if{
\section{User modelling} \label{sec:user_modelling}

Our approach relies extensively on careful user modelling,
and in particular, adheres to three key principles that guide the way in which we consider
users as decision-makers.

\textbf{Causal reasoning.}
Humans view the world around them through a causal lens,
and this is reflected in their behavior.
Thus, in domains where labeled data consists of human inputs,
there is no single, `true' underlying data generating process;
rather---all is in the eye of the beholder.
One implication of this is that different users may have different causal views of the world
(as demonstrated above).
Thus, even if two users are identical in their (true) preferences---they may differ in how they respond to \emph{change} (for example, a sudden decrease in $r$).
Hence, for providing \emph{robust personalization}, learning must consider users as \emph{causal decision-makers} \tocite.

\textbf{Uncertainty.}
As we will show,
causal modelling becomes effective once \emph{uncertainty} affects decisions
(this happens when there is an `unblocked' causal path from the source of uncertainty to the variable encoding user value).
It is only by modelling what users don't know---and how they handle this uncertainty---that robustness to change is made possible.
Modelling user-side uncertainty is important since
uncertainty is mediated by the system (as done by $r$ in the example above),
but more importantly, because reducing this uncertainty should be a primary goal.
Nonetheless, to the best of our knowledge, and despite being a fundamental element in all major theories of decision-making \tocite,
the notion of user uncertainty has been largely overlooked in machine learning
(e.g., click models \todo{elaborate}).
One of our goals in this work is therefore to advocate for modelling users
as \emph{decision-makers under uncertainty}.

\textbf{Bounded rationality.}
In most reasonable use cases,
users do not (or cannot) have access to the actual causal mechanics of the environment.
Thus, they must ground their reasoning on some approximation;
as decision-makers, this makes them \emph{boundedly-rational} \tocite. 
In our example, different users have different beliefs regarding $r$---whereas 
in reality it may simply be the output of a predictive model, or a simple aggregate statistic
(indeed, confusing correlation with causation is a notorious human fallacy \tocite).
Note that uncertainty alone does not make users non-rational; 
on the contrary---uncertainty is key to many rational decision models
(a prime example is the influential Random Utility Theory \tocite).
However, under rational behavior, uncertainty often results in increased variance,
whereas under irrational behavior, uncertainty results in bias---systematic,
predictable, and hence, exploitable---a longstanding hallmark of behavioral economics (\tocite).
Our approach makes use of this bias as it manifests
in how users causally reason about value and uncertainty to provide robustness.


\textbf{Comparison.} 
To see the merits of our approach,
consider by comparison the conventional approach for predicting user choices,
in which the goal is to learn personalized score functions $f_u(x)$
\todo{do we really want to use personalized here? maybe just $f(u,x)$, or even just $f(x)$}
such that
$f_u(x)>0$ whenever user $u$ chooses item $x$. 
Note that as a model of user choice behavior, this is akin to assuming rationality:
$x$ carries information regarding the intrinsic value of the item to $u$,
and $u$ chooses $x$ only if it provides non-negative utility
(as depicted by $f_u(x)$).
This approach can easily be extended to include uncertainty
(e.g., by including a additive noise term $f_u(x)+\varepsilon$, as done in discrete choice modelling \tocite)---but this mostly leads to enlarged variance.
System inputs such as relevance scores can also be included via
$f_u(x,r)$, but these simply act as additional features,
and do not capture their effect on how users perceive value.
Finally, and most importantly, the predictor $f_u(x)$ is innately \emph{correlative},
and so loses any of its guarantees once the input distribution changes.
}\fi



	%

\if
\orange{
	Recent advances in predictive modelling have percolated into recommendation systems and online platforms, leading to substantial improvements in their performance and to wide deployment. Such platforms attempt to generate personalized item recommendations by training on large scale offline data. As they rely on logged data, these systems often assume users are passive players, and rely on correlations between items, user attributes and decisions made in their predictive models \cite{krauth2020offline}.
	
	However, in reality users often have different beliefs regarding the generating process of the recommendations, and use these recommendations in various different ways. For example, some users might fail to recognize that a recommendation system is merely calculating correlations between attributes and actions, and imagine causal decisions being made by the recommendation engine. Alternatively, users might imagine that the system is trying to push forward some preferred content and modify its descriptions of it. Such beliefs could alter the decisions users make, and in return affect the recommendation system using those decisions as signals during training.
	
	The varying beliefs users might hold regarding the recommendation system is a direct result of the uncertainty they face when making decisions. Even if the system is transparent in the way it produces its recommendations, users might still believe that in reality a different unseen process is what governs the system's output. The key point is that even if the true world model is known, what matters is how the users observe the world and the system. Indeed, this realization is a foundation of the bounded rationality literature \cite{spiegler2020behavioral}.
	
	The simplifying assumptions made by existing recommendation systems often do not take into account the interaction between the system and different users \cite{mladenov2021recsim}. Such differences could be thought of as alternative world models, where each type of user assumes a different role for the recommendation system in her decision making process. Without accommodating these different beliefs, the system would not be able to learn a faithful mapping between users, items and decisions. 
	
	While assuming that users imagine a passive system might not hurt predictive performance on items sample i.i.d. from the same distribution observed during the offline training, in realistic scenarios where both the distribution of items and users shifts constantly it could hurt performance. Such performance drop could happen when there are spurious correlations in the training data, an association between some random variable and users' decision that holds no causal information about their action and might not hold when the distribution shifts in test time. For example, a user might believe that the recommendation system is pushing their products, giving them a higher ranking or placing them higher in a list. As such, she might choose not buy them regardless of their quality. Regardless if the system is pushing its content in such a way, what matters is that the user believes that way. Without knowing this causal user belief, the system might learn to associate its products with lower satisfaction, an empirical regularity that will not hold when faced with users that don't share this belief.
	
	One approach to modelling users' beliefs and the differences in their world models are causal graphs \cite{pearl2009causality}. Causal graphs are probabilistic graphical models, used to encode assumptions about the data-generating process. They can also be viewed as a blueprint of the algorithm by which Nature assigns values to the variables in the domain of interest. In our context, users define the world in which systems operate, meaning that user's causal graphs can be similarly viewed as a blueprint of Nature's algorithm. 
	
	In this work, we utilize causal graphs to crystallize the effect of users' beliefs, modeled as causal graphs, on recommendation systems. Specifically, we define a setup where it is possible to measure the effect of an alternative user belief on the types of recommendations provided by the system and therefore its predictive performance \ref{sec:causal_model}. We imagine different types of users, encode their beliefs with a causal graph, and discuss the effect such beliefs might have on a system that interacts with this user type.
	
	To demonstrate the power of our modelling approach on the predictive power of a recommendation system, we cast the problem of generalizing to unseen items as an out-of-distribution (OOD) generalization task (Section \ref{sec:learning}). In such a task, a system learns from multiple item distributions, and is tested on a new, unseen distribution of products. Leveraging recent advancements in \textit{invariant learning} \cite{wald2021calibration, veitch2021counterfactual}, we show that systems that take into account the causal beliefs of its users better generalizes to OOD examples (Section \ref{sec:exp}). Following that, we show that in realistic settings, where both users and items may vary in test-time, customizing predictions according to user beliefs is the best strategy. Finally, we discuss the power of causal modelling to represent user behavior and its potentially usefulness in dealing with users with diverse beliefs (Section \ref{sec:disc}).
}
\fi
\section{Related Work}
\label{sec:related}





\textbf{Causality and Recommendations.}
Formal causal inference techniques have been used extensively in many domains, but have only recently been applied to recommendations \cite{liang2016causal, wang2018deconfounded, bonner2018causal, zhang2021causal, wang2021deconfounded}. \citet{liang2016modeling} use causal analysis to describe a model of user exposure to items. Some work has also been done to understand the causal impact of these systems on behavior by finding natural experiments in observational data \cite{sharma2015estimating, su2016effect, schnabel2016recommendations}, and through simulations \cite{chaney2018algorithmic, schmit2018human}. \citet{bottou2013counterfactual} use causally-motivated techniques in the design of deployed learning systems for ad placement to avoid confounding. 
As most of this literature addresses selection bias and the effect of recommendations on user behavior \cite{bonner2018causal, zhang2021causal, wang2021deconfounded}, there is no work, as far as we know, that models boundedly rational agents interacting with a recommender system. Moreover, we are the first to propose modeling users' (mis)perceptions about the recommendation generation process using causal graphs.




\textbf{Bounded Rationality and Subjective Beliefs.}
The bounded rationality literature focuses on modelling agents that make decisions under uncertainty, without the ability to fully process the state of the world, and therefore hold subjective beliefs about the data-generating process.
\citet{eyster2005cursed} defined \textit{cursed beliefs}, which capture an agent's failure to realize that his opponents' behavior depends on factors beyond those he is informed of. Building on \citet{esponda2016berk}, who modelled equilibrium beliefs under misspecified subjective models, \citet{spiegler2020behavioral} used causal graphs to analyze agents that impose subjective causal interpretations on observed correlations. 
This work lays the foundation upon which we model users here, and has sprouted many interesting extensions \cite{eliaz2020cheating, eliaz2021strategic}.


\textbf{Causality and Invariant Learning.}
Correlational predictive models can be untrustworthy~\cite{jacovi2021formalizing}, and latch onto spurious correlations, leading to errors in OOD settings~\cite{mccoy2019right, feder2021causalm, feder2021causal}. This shortcoming can potentially be addressed by a causal perspective, as knowledge of the causal relationship between observations and labels can be used to mitigate predictor reliance on them~\cite{buhlmann2020invariance,veitch2021counterfactual}. In our experiments, we learn a representation that is invariant to interventions on the `environment' $\env$, a special case of an invariant representation \cite{arjovsky2019invariant, krueger2020out, bellot2020generalization}. Learning models which generalize OOD is a fruitful area of research with many recent developments \cite{magliacane2018domain,heinze2018invariant,peters2016causal, subbaswamy2019preventing, ben2021pada, wald2021calibration}. Recently, \citet{veitch2021counterfactual} showed that the means and implications of invariant learning depend on the data's true causal structure. Specifically, distinct causal structures require distinct regularization schemes to induce invariance. 
\section{Modelling Approach}
\label{sec:setting}

\subsection{Learning Setting}
In our setting, data consists of users, items, and choices.
Users are described by features $u \in \R^{d_u}$,
and items are described by two types of features:
intrinsic item properties, $x \in \R^{d_x}$
(e.g., movie genre, plot synopsis, cast and crew),
and information provided by the platform, $r \in \R^{d_r}$
(e.g., recommendation score, user reviews).
We will sometimes make a distinction between features that are available to users, and those that are not;
in such cases, we denote unobserved features by $\xbar$,
and with slight abuse of notation, use $x$ for the remaining observed features
(we assume $r$ is always observed).
Choices $y \in \{0,1\}$ indicate
whether a user $u$ chose to interact (e.g., click, buy, watch) with a certain item $(x,r)$.
Tuples $(u,x,r,y)$ are sampled iid from
certain unknown joint distributions, which we define next.

As we are interested in robustness to distributional change,
we follow the general setup of \emph{domain generalization}
\cite{blanchard2011generalizing, koh2020wilds, ben2021pada, wald2021calibration}
in which there is a collection of environments, denoted by a set $\envs$,
and each environment $\env \in \envs$ defines a 
different joint distribution $
\dist^{\env}$ over $(u,x,r,y)$.
We assume there is training data available
from a subset of $K$ environments,
$\envstrn = \{\env_1,\dots,\env_K\} \subset \envs$,
with datasets $\smplst_k = \{(u_{ki},x_{ki},r_{ki},y_{ki})\}_{i=1}^{m_k}$
drawn i.i.d from the corresponding $\dist^{\env_k}$.
We denote the
pooled training distribution by $\disttrn = \cup_{\env \in \envstrn} \dist^e$
and the pooled training data by $S=\cup_k S_k$ with $m=\sum_k m_k$.

Our goal is to learn a robust predictive model
$\yhat=f(u,x,r; \theta):=f_u(x,r; \theta)$ with parameters $\theta$;
the $f_u$ notation will be helpful in our discussion of robustness
as it emphasizes our focus on individual users.
We now turn to define the precise type of robustness that we will be seeking.


\textbf{Robustness via causal graphs.}
The type of robustness that we would like our model to satisfy is \emph{counterfactual invariance} (CI) \cite{veitch2021counterfactual}. Denoting $x(e), r(e)$ as the counterfactual features that would have been observed had the environment been set to $e$, this is defined as:
\begin{definition}
A model $f_u$ is 
CI if $\forall e,e'\in{\envs}$
it holds a.e. that
$f_u(x(e'),r(e');\theta) = f_u(x(e),r(e);\theta)$. 
\end{definition}

The challenge in obtaining CI predictors is that at train time we only observe a subset of the environments, $\envstrn \subset \envs$, while CI requires independence to hold for \emph{all} environments $\env \in \envs$.
To reason formally about the role of $\env$ in the data generating process, and hence about the type of distribution shifts under which our model should remain invariant,
it is common to assume that a causal structure
underlies data generation \cite{heinze2018invariant, arjovsky2019invariant}.
This is
often modeled as a (directed) causal graph \cite{pearl2009causality};
robustness is then defined as insensitivity of the predictive model
to changes (or `interventions') in the variable $\env$,
which can trigger changes in other variables
that lie `downstream' in the graph.
To encourage robustness,
a common approach is to construct a learning objective
that avoids spurious correlations by
enforcing certain conditional independence relations to hold,
e.g., via regularization (see \S \ref{sec:learning}).
The question of \emph{which} relations are required can be answered
by examining the graph
and the conditional independencies it encodes
(between $\env,x,r,y$, and $ \yhat$).
Unfortunately, inferring the causal graph is in general hard;
however, determining the `correct' learning objective
may require only partial
information regarding the graph.
We will return to the type of information we require for our purposes,
and the precise ways in which we use it,
in \S \ref{sec:learning}.

\subsection{Users as decision makers}

Focusing on relevance prediction,
at the heart of our approach lies the observation that
what underlies the generating process of data, and in particular of labels,
\emph{is the way in which users causally perceive the environment}.
In this sense, users \emph{are} the causal mechanism,
and their causal perceptions manifest in their decision-making.
Operationally, we model users as acting on the basis of
\emph{individualized causal graphs} \cite{spiegler2020behavioral}
that define how changes in one variable
propagate to influence others,
and ultimately---determine choice behavior.
This allows us to anticipate, target, and account for
sources of spuriousness.

\textbf{Rational users.}
To see how modeling users as causal decision-makers can be helpful,
consider first a conventional `correlative'
approach for training $f_u$,
e.g. by minimizing the loss of a corresponding score function $v_u(x,r)=v(u,x,r)$,
and predicting via 
$\yhat = \argmax_{y\in\{0,1\}} y v_u(x,r) = \one{v_u(x,r) > 0}$.
From the perspective of user modeling,
$v_u$ can be interpreted as a personalized `value function';
this complies with classic Expected Utility Theory (EUT) \cite{thurstone1927law},
in which users are modeled as rational agents
acting to maximize (expected) value, and under \emph{full information}.\footnote{For simplicity, here and throughout we consider deterministic valuations, although this is not necessary.}
From a causal perspective, this approach is equivalent to assuming a graph in which all paths from $e$ to $y$ are blocked by $x,r$---which is akin to assuming no spurious pathways, and so handicaps the ability to avoid them.

\textbf{Boundedly-rational users.}
We propose to model users as boundedly-rational decision-makers,
under the key assertion that users' decisions take place
under inherent \emph{uncertainty}.
Uncertainty plays a key role in how we, as humans, decide:
our actions follow not only from what we know,
but also from how we account for what we don't know.
Nonetheless, and despite being central to most modern theories of decision making \cite{kahneman1982judgment}---and despite being a primary reason for why users turn to online informational services like recommendation systems in the first place---explicit modeling of user-side uncertainty is currently rare within machine learning \cite{plonsky2019predicting,apel2020predicting,raifer2021designing};
here we advocate for its use.

Our modeling approach acknowledges that users \emph{know} some features are unobserved, and that this influences their actions.
Here we demonstrate how this relates to robust learning
through an illustrative example using a particular behavioral model,
though as we will show, our approach applies more broadly.
Consider a user shopping online for a vintage coat,
and considering whether to buy a certain coat.
The coat's description includes
several intrinsic properties $x$
(e.g., the coat's material),
as well as certain platform-selected information $r$
(e.g., a stylized photo of a vintage-looking coat).
The user wants to make an informed decision,
but knows some important information, $\xbar$, is missing
(e.g., the year in which the coat was manufactured).
If she is concerned about buying a modern knockoff
(rather than a truly vintage coat),
how should she act?
A common approach is to extend EUT to support uncertainty
by modelling users as integrating subjective beliefs
about unobserved variables, $p_u(\xbar|x,r,\env)$,
into a conditional estimate of value, $\vtilde_u(x,r|\env)$,
over which choices $y$ are made:
\begin{equation}
\label{eq:value_decomp}
\vtilde_u(x,r|\env) 
= \sum_\xbar v_u(x,\xbar,r) p_u(\xbar|x,r,\env), \qquad \quad
y= \one{\vtilde_u(x,r|\env)>0}
\end{equation}
Here, $p_u(\xbar|x,r,\env)$ describes a user's (probabilistic) belief regarding the conditional likelihood of each $\xbar$ 
(is the vintage-looking coat truly from the 60's?),
and $v_u(x,\xbar,r)$ describes the item's value to the user
\emph{given} $\xbar$ 
(if the coat really is from the 60's---how much is it worth to me?).
Importantly, note that uncertainty beliefs $p_u(\xbar|x,r,\env)$
can be environment-specific
(i.e., the degree of suspicion regarding knockoffs can vary across retailers).
In turn, value estimates $\vtilde$ and choices $y$ can also rely on $\env$
(note that $\vtilde$ can also be interpreted as the conditional
expected value, $\vtilde_u(x,r|e) = \expect{p_u(\cdot|\env)}{v|x,r}$).
Since $y$ is a deterministic function of $\vtilde$ (Eq. \eqref{eq:value_decomp}),
for clarity (and when clear from context)
we will ``skip'' $\vtilde$ and refer to $y$ as a direct function of $x,r$, and $\env$. 

\begin{figure*}[t!]
    \centering
    \includegraphics[width=\textwidth]{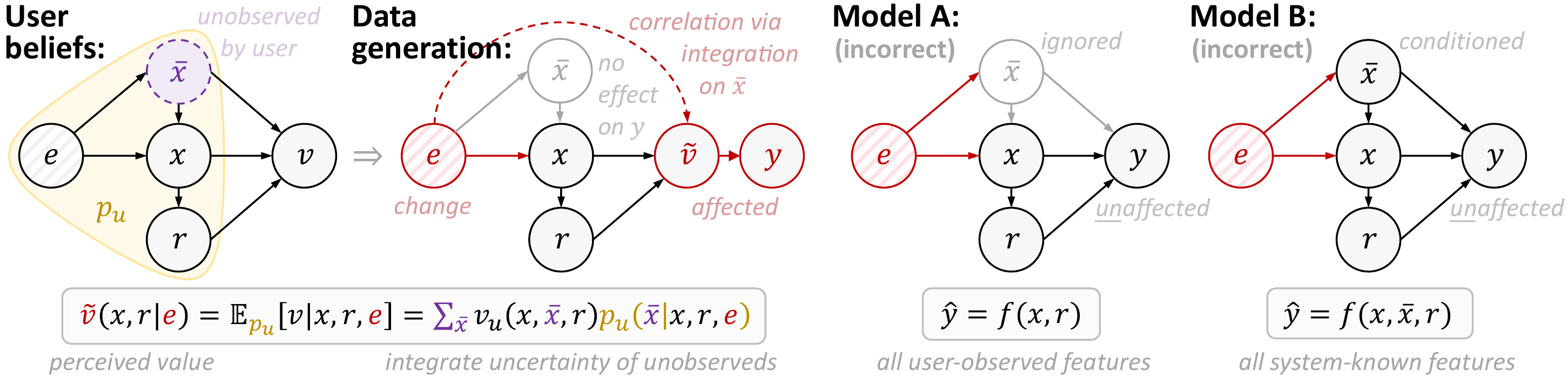}
    \caption{
    \textbf{(Left)} User beliefs and resulting data generation process of $y$. 
    The user does not observe $\xbar$, but knows it is missing.
    To compensate for this uncertainty, the user integrates over $\xbar$ w.r.t. probabilistic beliefs $p_u$, which can depend on $\env$;
    this forms her perceived value $\vtilde$, which determines her choice $y$. Integrating uncertainty can introduce correlation between $\env$ and $\vtilde$ (and hence $y$) through $p_u(\cdot|\env)$.
    Note $y$ does not depend on \emph{instances} of $\xbar$.
    \textbf{(Center)} Learning a predictive model $f(x,r)$ using only features observed by the user.
    Assuming that $\xbar$ can be discarded creates the impression that
    information cannot flow from $\env$ to $y$,
    implying (wrongly) that a {\naive}ly trained $f(x,r)$ would be robust.
    In practice, such an $f$ might incorrectly use $x,r$ to compensate for variation in $\env$.
    \textbf{(Right)} Learning a predictive model $f(x,\xbar,r)$ using all features available to the system.
    Assuming $\xbar$ affects $y$ implies (wrongly) that all paths from $\env$ to $y$ are blocked.
    {\Naive}ly training $f$ will likely use variation in $\xbar$ to explain $y$; this might improve performance on observed $\env$,
    but will not generalize to others.
    }
    \label{fig:user_graphs}
\end{figure*}

\textbf{Causal user graphs.}
One interpretation of Eq. \eqref{eq:value_decomp} is that users cope with uncertainty by employing \emph{causal reasoning} \citep{spiegler2020behavioral}, 
this aligning with a predominant approach in the cognitive sciences
that views humans as acting based on `mental causal models' \citep{sloman2005causal}.
Here we follow \citep{spiegler2020behavioral} and think of 
users as reasoning through personalized \emph{user causal graphs},
denoted $G_u$.
The structure of $G_u$ expresses $u$'s causal beliefs---namely
which variables causally affect others---and 
its factors correspond to the conditional terms 
($p_u$ and $v_u$) in Eq. \eqref{eq:value_decomp}.
A key modeling point is that
users can vary in their causal perceptions;
hence, different users may have different graphs that encode different conditional independencies,
these inducing different simplifications of the conditional terms.
For example, a user that believes movies with a five-star rating
($r$) are worthwhile regardless of their content ($x$)
would have $v_u(x,\xbar,r)$ reduced to $v_u(\xbar,r)$,
since $v \Vbar x | r$;
meanwhile, a user who, after reading a movie's description ($x$),
is unaffected by its rating ($r$),
would have $v_u(x,\xbar)$ instead, since $v \Vbar r | x$.

\subsection{User behavior and spurious correlations}
We are now ready to make the connection to learning.
Recall that our goal is to learn a predictor $f$
that is unaffected by spurious correlations,
and that these can materialize if
some mechanism creates an association between $\env$ and $y$;
we will now see how user behavior can play such a role.
Continuing our illustrative example,
assume that the beliefs of our 
boundedly-rational user (who chooses via Eq. \eqref{eq:value_decomp})
are encoded by the leftmost diagram in Fig. \ref{fig:user_graphs} (`User beliefs'). 
The diagram does not show an edge between $\env$ and $y$.
However, and crucially,
the user behaves `as if' there actually was an edge:
by accounting for uncertainty via integration,
$\xbar$ is effectively `removed' from the indirect path $\env {\rightarrow} \xbar {\rightarrow} y$,
which results in a direct connection between $\env$ and $y$.
The corresponding data-generation process of choices $y$
is illustrated in Fig. \ref{fig:user_graphs} (`Data generation').
This supports our main argument:
by making decisions, users can \emph{generate} spurious correlations in the data---here, by accounting for uncertainty.



The above has concrete implications on learning.
First,
it shows how conventional learning approaches can fail.
On the one hand, since users observe only $x$ and $r$,
one reasonable approach would be to discard $\xbar$ altogether, and train
a predictor $f_u(x,r;\theta)$ in hopes of mimicking user choice behavior.
This means learning as if there is no edge between $\env$ and $y$
(Fig. \ref{fig:user_graphs}, `Model A').
Under this (incorrect) assumption, for $f_u$ to be robust to variation in $\env$,
it would suffice to train using a conventional approach, e.g., vanilla ERM.
But the user \emph{knows} $\xbar$ exists, and by integrating beliefs,
relies on this for producing $y$---importantly, in a way that \emph{does} depend on $\env$.
This makes learning $f_u(x,r;\theta)$ prone to using $x$ and $r$ to compensate for the
constant effect of each train-time environment $\env_k$ on $y$.
By definition, once the environment changes,
a {\naive}ly trained $f_u$ cannot account for 
the new (residual) effect of $\env$ on $y$.


Conversely,
the system may choose to learn using all information that is available to it,
namely train a predictor $f_u(x,\xbar,r;\theta)$ (Fig. \ref{fig:user_graphs}, `Model B').
This make sense if the goal is in-distribution (ID) generalization.
But for out-of-distribution, this creates an illusion
that conditioning on $x,\xbar,r$ will block all paths from $\env$ and $y$, again making it tempting to (wrongly) conclude that applying ERM would suffice for robustness.
However, because the system does observe instances of $\xbar$
(note it still exists in the data generating process),
learning can now erroneously use the variation in $\xbar$ to explain $y$,
whereas the true $y$ does not rely on specific instantiations of $\xbar$.
As a result, a {\naive}ly learned $f_u$ will likely overfit to training distributions in $\envstrn$, and may not generalize well to new environments.

Second, 
the awareness to how users account for uncertainty
suggests a means to combat spuriousness.
Eq. \eqref{eq:value_decomp} shows that $y$ depends on $x,r$, and $e$;
hence, since our goal is to discourage the dependence of $\yhat$ on $\env$,
it follows that (i) functionally, $f_u$ should not depend on $\xbar$, 
but (ii) $f_u$ should be learned in a way that controls for (conditional) variation in $\env$.
In our example, this manifests in the role of $\xbar$:
at \emph{train-time} use $\xbar$ (perhaps indirectly)
to learn a function that at \emph{test-time} does not rely on it
(c.f. Model B which uses $\xbar$ for both train and test,
and Model A which does not use $\xbar$ at all).\footnote{The careful reader will notice that Fig. \ref{fig:user_graphs} reveals how spuriousness can arise, but not yet how it can be handled.
Indeed, the latter requires additional structure
(which the figure abstracts)
that is described in Sec. \ref{sec:learning}.}
Since the precise way in which $\env$
relates to other variables is determined by the user graph $G_u$
(which determines conditional independencies),
knowledge of the graph should be useful in
promoting invariance.
In the next section we describe what knowledge is needed,
and how it can be used for invariant learning.

\section{Learning With Causal User Models}
\label{sec:learning}

Our approach to robust learning is based on
regularized risk minimization,
where regularization acts to discourage variation in predictions across environments \cite{veitch2021counterfactual, wald2021calibration}.
Our learning objective is:
\begin{equation}
\label{eq:learning_objective}
\argmin_{f \in F} \loss(f; S)
+ \lambda \reg(f;\smplst_1,\dots,\smplst_K)
\end{equation}
where $F$ is the function class,
$\loss$ is the average loss w.r.t an empirical loss function (e.g., log-loss),
and $\reg$ is a data-dependent regularization term with coefficient
$\lambda$.
In our approach, the role of $\reg$ is to penalize $f$
for violating certain statistical independencies;
the question of \emph{which} independencies should be targeted---and hence the precise form that $R$ should have---can be answered by the underlying causal graph \citep{veitch2021counterfactual}.
Knowing the full user graph (e.g. detailed causal relations between different features within $x$, such as whether the genre of a movie is a cause for its production quality) can certainly help 
, but relying on this (and at scale) is impractical.
Luckily, as we show here,
coarse information regarding the graph can be translated into necessary conditions for distributional robustness, and in \S \ref{sec:exp} we will see that these can go a long way towards learning robust models in practice.

Note that by choosing to promote robustness through regularization,
our approach becomes 
agnostic to the choice of function class $F$ (although the graph may also be helpful in this\footnote{E.g., if we aim to learn parameterized predictors on the basis of Eq. \eqref{eq:value_decomp}, graphs can help discard dependencies.}).
We also need not commit to any specific behavioral choice model.
This allows us to abstract away from the particular behavioral mechanism that generates spuriousness (e.g., integration of $\xbar$),
and consider general relations between $\env$ and $y$
(e.g., selection or common cause); we make use of this in our experiments.



\textbf{Regularization schemes.}
We focus on two methods for promoting statistical independence:
MMD \cite{gretton2012kernel}, which we present here;
and CORAL \cite{sun2016deep}, which we describe in Appendix \ref{sec:app_learning}
(we use both in our experiments).
The MMD regularizer applies to models $f$ that can be expressed as a
predictor $h$ applied to a (learned) representation mapping $\phi$,
i.e., $f=h \circ \phi$ (note $h$ can be vacuous).
MMD works by encouraging the (empirical) distribution of representations for each environment $e_k$
to be indistinguishable from all others;
this is one way to express the independence test for $\yhat$ and $e$ \cite{veitch2021counterfactual}.
In our case, for a single $e_k$, MMD is instantiated as:
\begin{equation}
\MMD(\Phi_k, \Phi_{-k}),   \,\,\, \text{where} \quad
\Phi_k = \{\phi(x):x \in S_k\}, \,\,
\Phi_{-k} = \{\phi(x):x \in S \setminus S_k\}.
\end{equation}
As we show next, the precise way in which MMD is used for regularization depends on the graph.

\subsection{User graph classes: causal vs. anti-causal}

Following our example in Fig.\ref{graph-causal-1}, consider users of two types:
a `causal' user $\uc$ that believes value is an \emph{effect} of an item's description (i.e., $D^e(x,r,y \mid u=\uc)$ is entailed by the graph $x,r \textcolor{blue}{\rightarrow} y$ for each $e\in{\envs}$),
and an `anti-causal' user $\uac$ that believes the item's value \emph{causes} its description
(i.e., $D^e(x,r,y \mid u=\uac)$ is entailed by $x,r \textcolor{red}{\leftarrow} y$ respectively).\footnote{Technically, anti-causal users have $x,r\textcolor{red}{\leftarrow} \vtilde$, 
and $y$ as a function of $\vtilde$,
but we use $x,r \textcolor{red}{\leftarrow} y$ for consistency.}
Our next result shows that:
(i) $\uc$ and $\uac$ require \emph{different} regularization schemes;
but (ii) the appropriate scheme is fully determined 
by their type---\emph{irrespective} of any other properties of their graphs.
Thus, from a learning perspective, it suffices to know
which of two classes a user belongs to: causal, or anti-causal.
\begin{proposition}
\label{prop:class}
Let $f$ be a CI model and assume $y$ and $\env$ are confounded (e.g., $\env{\rightarrow}y$ exists),
then:
    \setlist{nolistsep}
    \begin{enumerate}[noitemsep,leftmargin=0.65cm,label=(\arabic*)]
    \item $f_{\uc}$ must satisfy
    $\prob{D^e}{f_{\uc}(x,r)} = \prob{D^{e'}}{f_{\uc}(x,r)}\,\,\, \forall e,e'\in{\envs}$
    \item $f_{\uac}$ must satisfy
    $\prob{D^e}{f_{\uac}(x,r) \mid y} = \prob{D^{e'}}{f_{\uac}(x,r) \mid y} \,\,\, \forall e,e'\in{\envs}$, $y\in{\{0, 1\}}$
    \end{enumerate}
On the other hand, $f_\uc$ need not necessarily satisfy (2),
and $f_\uac$ need not necessarily satisfy (1).
\end{proposition}
If we fail to enforce these constraints during learning, then we will not learn a CI classifier. On the other hand, enforcing unnecessary constraints (e.g.,  requiring both conditions hold for $f_{\uc}$ and $f_{\uac}$) restricts our hypothesis class and hence limits performance.
The proof follows directly from \cite{veitch2021counterfactual} (under technical assumptions; see Appendix \ref{sec:app_theory}).
The distinction between causal and anti-causal is fundamental in causality \citep{scholkopf2012causal};
for our purposes, it prescribes the appropriate regularization.
\begin{corollary}
\label{cor:reg}
For any user $u$,
to encourage $f_u(x, r)$ to be invariant to changes in $\env$, set:
\begin{equation}
\label{eq:reg_mmd}
R(f;\smplst) =
\begin{cases}
\sum_k \MMD(\Phi_{k, u}, \Phi_{-k, u}) & u \textup{ is causal} \qquad\qquad\qquad \darkgray{\text{(marginal MMD)}} \\
\sum_y \sum_k \MMD(\Phi_{k, u}^{(y)}, \Phi_{-k, u}^{(y)}) & u \textup{ is anti-causal} \qquad\qquad\, \darkgray{\text{(conditional MMD)}}
\end{cases}    
\end{equation}
\vspace{-2mm}
where $\Phi_{k,u}, \Phi_k^{(y)}$ includes the subset of examples with user $u$ and label $y$, respectively.
\end{corollary}
When learning over multiple users (Eq. \eqref{eq:learning_objective}),
the operational conclusion is that all users of the same class---regardless of their specific graphs---should be regularized in the same manner, as in
Eq. \eqref{eq:reg_mmd}.
\subsection{User graph subclasses: inter-feature relations} \label{sec:subclasses}

Consider now two users that are of the same class (i.e., causal or anti-causal),
but perceive differently the causal relations between $x$ and $r$:
a \emph{believer}, $\ubel$, who believes recommendations follow from the item's attributes;
and a \emph{skeptic}, $\uskep$, who presumes that the system reveals item attributes to match a desired recommendation
(see Fig.\ref{graph-causal-2}).
Our main result shows that even if both users share the same objective preferences and hence exhibit the same choice patterns---to be \emph{optimally} invariant, each user may require her own, independently-trained model (though with the same regularization).
\begin{proposition}
\label{prop:subclass}
Let $\ubel,\uskep$ be two users of the same class (i.e., causal or anti-causal)
but of a different subclass (i.e., believer and skeptic, respectively).
Even if there is a single predictor $f$ which is optimal for the pooled distribution $\disttrn$,
each user can have a different optimal CI predictor.
\end{proposition}
Proof is in Appendix \ref{sec:app_theory}.
Prop. \ref{prop:subclass} can be interpreted as follows:
Take some $u$, and `counterfactually' invert the edges between $x$ and $r$.
In some cases, this 
will have no effect on $u$'s behavior under 
$\envstrn$,
and so any $f$ that is optimal in one case will also be optimal in the other.
Nonetheless, for optimality to carry over to \emph{other} environments---different predictors may be needed.
This is since
each causal structure implies a different interventional distribution, 
and hence a different set of CI predictors:
e.g., in $G_{x \textcolor{brown}{\leftarrow} r}$, the v-structure $e {\rightarrow} x \textcolor{brown}{\leftarrow} r$ suggests that an invariant predictor may depend on $r$, yet in $G_{x \textcolor{teal}{\rightarrow} r}$ it cannot. 
If the sets of CI predictors do not intersect, then necessarily there is no single optimal model. 

The practical take-away is that even if two different users exhibit similar observed behavioral patterns (e.g., differences in their graphs are not expressed in 
the data), whether they are skeptics or believers has implications
for robust learning;
Prop. \ref{prop:subclass} considers an extreme case.
Luckily, for data with mixed subclasses,
having multiple training environments enables us to nonetheless learn invariant predictors, e.g., 
by partitioning the data by user subclass---whether inferred or assumed---and learning a different model for each subclass
(with regularization determined by class).


\textbf{Graph knowledge: inference vs. beliefs.}
Formally, both Prop. \ref{prop:class} and Prop. \ref{prop:subclass} require that we precisely know each user's
class and subclass, respectively,
which amounts to inferring the directionality of a subset of edges.
In principle, this can be done via experimentation
(e.g., using focused interventions such as A/B tests)
or from observational data (using simple conditional independence tests, e.g., \cite{spirtes1991algorithm, glymour2019review}) \cite{shalizi2013advanced}.
While this is certainly easier than inferring the entire graph,
orienting edges can still be challenging or expensive.
Nonetheless, Prop. \ref{prop:class} can still be 
practical useful when there is 
good \emph{domain knowledge} regarding user classes,
at either the individual or population level:
if the learner has certain prior beliefs regarding users' causal perceptions, Eq. \eqref{eq:reg_mmd}
provides guidance as to
how to devise the learning objective:
what regularization to apply (Prop. \ref{prop:class}),
and how to partition the data (Prop. \ref{prop:subclass}).
Our experiments in Sec. \ref{sec:exp},
which we present next, are designed under this perspective.

\if
\green{\rule{\textwidth}{3pt}}


We focus on settings where we can learn predictive models that take user beliefs into account. Building on recent advances in invariant learning \cite{arjovsky2019invariant,wald2021calibration}, we cast the problem of a user's causal beliefs affecting her choices (\S \ref{sec:setting}), as an OOD generalization problem. The focus on robustness to changes in $z$ allows us to pinpoint the failures that are caused by not directly modelling user behavior, and to clearly demonstrate the benefit from learning user-aware predictors. Our objective is to learn a predictor $f(x(z),r)$ for user $u=u^*$ that is independent of $z$, meaning that $f(x(z'),r) = f(x(z),r) \text{, } \forall z \in Z$. The reason for training such a predictor is that when the distribution of $p(z,y)$ changes, it will still be able to make correct predictions. To learn such a model, it was recently argued that knowing the true causal model generating the data is beneficial \cite{veitch2021counterfactual}. 


As we will show here, in our use-case this entails that if the system learns a user-type-specific predictor, where the training objective for each predictor takes into account the corresponding causal graph, then it can be more stable across distribution shifts. There are two important learning implications for a given graph: (1) the regularization scheme should match the causal structure between the input features and the label, where a causal $x,r \textcolor{blue}{\rightarrow} y$ and anti-causal $x,r \textcolor{red}{\leftarrow} y$ links will induce a marginal and conditional (on $y$) distributional penalty, respectively; (2) causal relationships between input variables (i.e. $x \textcolor{blue}{\rightarrow} r, x \textcolor{red}{\leftarrow} r$) might entail a different relationship between $f(x,r)$ and $z$, such that to be CI to $z$ we would get that $f_{x \textcolor{blue}{\rightarrow} r}(x,r) \neq f_{x \textcolor{red}{\leftarrow} r}(x,r)$, regardless of their causal relationship with the label.



\subsection{Invariant Learning with User Graphs} 


Following the example analyzed in \S \ref{sec:setting}, we first describe a case where users differ by their belief regarding the causal direction between an item's observed attributes $x$ and their actions $y$. Specifically, we divide users into two cases: \textit{causal} and \textit{anti-causal}, where \textit{causal} indicates the belief that item features cause user satisfaction \cite{scholkopf2012causal}. 

Dividing users into \textit{causal} and \textit{anti-causal} can be motivated from several perspectives. First, following the example in \S \ref{sec:setting}, it is intuitive to understand how these two types of users might make different decisions based solely on their beliefs \cite{spiegler2020behavioral}. Second, in the domain generalization literature, the distinction between \textit{causal} and \textit{anti-causal} learning problems gives rise to learning objectives that can take each type into account through a distinct regularization scheme \cite{veitch2021counterfactual}. Finally, using the graphical models literature we can distinguish between two different graphs using conditional independence tests only if they are not equivalent \cite{shalizi2013advanced}, a property that holds between the \textit{causal} and \textit{anti-causal} cases.

In in-distribution settings,
relying on spurious features, such as a tendency towards some items, might not hurt predictive performance. However, when such a tendency changes (i.e. when tastes change) predictive models can make arbitrarily large mistakes if they rely on such features \cite{arjovsky2019invariant}. Specifically, if new items are being introduced into the system with different $x$'s and corresponding $r$'s, the relationship between $z$ (i.e. the time the item was introduced) and user actions $y$ might induce a spurious correlation that will hurt performance. This type of relationship between $z,x,r,y$ is encoded in Figure \ref{graph-causal}(a).

\begin{figure}[ht]
    \centering
    \caption{Simplified graphs describing different user-types in our analysis. Users are either \textcolor{blue}{\textit{causal}} or \textcolor{red}{\textit{anti-causal}} w.r.t: (a) link between their satisfaction $y$ and observed product attributes $x$. (b) link between the system's recommendations $r$ and $x$.
    \todo{add (a) and (b) labels to graphics}
    \todo{switch from $z$ to $\env$}
    }
    \label{graph-causal}
    \begin{subfigure}{0.4\textwidth}
        \begin{tikzpicture}[node distance = 1.5cm]
        \node[circle, draw, text centered] (x) {$x,r$};
        \node[circle, draw, right of = x, text centered] (y) {$y$};
        \node[circle, draw, left of = x, text centered] (z) {$z$};

        \draw[blue, ->, line width = 1,transform canvas={yshift=1.5mm}] (x) -- (y);
        \draw[red, ->, line width = 1,transform canvas={yshift=-1.5mm}] (y) -- (x);
        \draw[->, line width = 1] (z) -- (x);
        \draw[<->, line width = 1, dash dot] (z) to[out=-90, in=-90, distance=0.5cm] (y);
        \end{tikzpicture}
    \end{subfigure}
    \begin{subfigure}{0.4\textwidth}
    \begin{tikzpicture}[node distance = 1.5cm]
        \node[circle, draw, text centered] (R) {$r$};
        \node[circle, draw, below of = R] (X) {$x$};
        \node[ circle, draw, right of = R, text centered] (y) {$y$};
        \node[circle, draw, left of = R, text centered] (Z) {$z$};

        \draw[->, line width = 1] (X) -- (y);
        \draw[blue, ->, line width = 1,transform canvas={xshift=1.5mm}] (X) -- (R);
        \draw[red, ->, line width = 1,transform canvas={xshift=-1.5mm}] (R) -- (X);
        \draw[->, line width = 1] (Z) -- (X);
        \draw[->, line width = 1] (R) -- (y);
        \draw[<->, line width = 1, dash dot] (Z) to[out=-90, in=-90, distance=2.1cm] (y);
    \end{tikzpicture}
    \end{subfigure}

\end{figure}



Alternatively, even if users share the beliefs regarding the causal link between $x,r$ and $y$, if there are differences in the perception of the link between $x$ and $r$, optimality might still require us to learn a user-specific model. Concretely, we can also consider such a case, depicted in Figure \ref{graph-causal}(b). In this case, all users imagine that $x,r \rightarrow y$, but differ w.r.t the imagined link between recommendations and item attributes. Here, a \textit{causal} user is one who thinks that recommendations are caused by the item's attributes, and an \textit{anti-causal} user believes that the observed attributes are chosen by the system according to its desired recommendation.



\subsection{Learning Implications of User Graphs}

As illustrated in Figures \ref{graph-causal}(a) and \ref{graph-causal}(b), for both the \textit{causal} and \textit{anti-causal} users, the relationship between $z$ and user actions $y$ can be reflected in the product attributes. To disentangle the effect of $z$ on the system's predictions, it would require a classifier $f(x,r)$ that satisfies the appropriate conditional independencies. Following \citet{veitch2021counterfactual}, we can find these independencies by examining the graph that describes the data generating process of our data. Depending on the graph, we arrive at different independencies which in turn induce different training procedures.



Our use of the causal graph for learning robust predictors is therefore rather limited. In fact, the only property we need to get from the graph is the conditional independence that should hold between our classifier $f(x,r)$ and $z$ such that we are independent of $z$. Then, to enforce such independence during model training, we apply a distributional penalty that enforces the desired property \cite{greenfeld2020robust, veitch2021counterfactual}. In our experiments (\S \ref{sec:exp}), we use the \emph{MMD} \cite{gretton2012kernel} and \emph{CORAL} \cite{sun2016deep} penalties to enforce independence between our predictor $f(\cdot)$ and $z$.
In the \textit{causal} and \textit{anti-causal} (w.r.t the label $y$) cases, this entails the following loss function: 

\begin{equation}
    \mathcal{L} = l_{CE} + \lambda \cdot l_{dist}
\end{equation}
\vspace{-7mm}

\vspace{-7mm}
\begin{align*}
    l_{dist} = \begin{cases}
    \text{MMD}(P(f(x,r),z),(f(x,r),z')) & \textbf{ (c)} \\
    \sum\limits_{y \in Y} \text{MMD}(P(f(x,r),z,y),(f(x,r),z',y)) & \textbf{ (a-c)} 
    \end{cases}
\end{align*}
For a detailed description of $\text{MMD}(f(x,r),z)$ and $\text{CORAL}(f(x,r),z)$, see Appendix \ref{sec:app_learning}. 


In Figure \ref{graph-causal-2}, we consider two causal models $G_{x \textcolor{blue}{\rightarrow} r}$ and $G_{x \textcolor{red}{\leftarrow} r}$, with edges $x \textcolor{blue}{\rightarrow} r$ and $x \textcolor{red}{\leftarrow} r$ respectively, denoting different perceptions users may have of the link between $x$ and $r$. Interestingly, while the risk minimizing counterfactually-invariant classifiers change according to the \textit{causal} user model, we may learn them using the same regularization scheme (marginal, or \textit{causal}). We state this as follows.

\begin{proposition}
Let $P_{x \textcolor{blue}{\rightarrow} r}, P_{x \textcolor{red}{\leftarrow} r}$ be distributions entailed by $G_{x \textcolor{blue}{\rightarrow} r}, G_{x \textcolor{red}{\leftarrow} r}$ accordingly. Then if $y$ and $z$ are not subject to selection (but may be confounded) then for both models, any invariant predictor must satisfy $f(x, r) \indep z$.
\end{proposition}
However, the crucial point here is that even though the same learning algorithm should be applied for users of both types, learning a separate model for each user type is still beneficial. That is because in general $P_{x \textcolor{blue}{\rightarrow} r}(y \mid x, r) \neq P_{x \textcolor{red}{\leftarrow} r}(y \mid x, r)$.
There are two benefits for separating a given dataset into multiple user-types: (1) learning a personalized model is better in loss terms for each user on its own; (2) even if we look at average test loss and the mixture between user types stays the same as in training, the robust model should be better if we split according to user type. We elaborate and illustrate this with a numerical example Appendix \ref{sec:app_synthetic}.
Intuitively, the different causal structures imply different interventional distributions under the two models, and hence a different set of CI predictors. For instance, in $G_{x \textcolor{red}{\leftarrow} r}$ the $V$ structure $z \rightarrow x \textcolor{red}{\leftarrow} r$ suggests that an invariant predictor may depend on $r$, yet in $G_{x \textcolor{blue}{\rightarrow} r}$ it cannot. 

\textbf{Identifying User-Types.}
While in certain cases a system might be able to know its users' beliefs, in general this is probably too strong of an assumption to make. A perhaps more realistic scenario is where there are different user-types, each having different beliefs, but we don't know \textit{ex-ante} and therefore can't assign each user to a specific causal model. In such a case, the system might have to discover the true causal graph or the desired conditional independence, and recommend given the predicted graph. Discovering user-types can generally be done through focused interventions such as A/B testing on recommendations, or through conditional independence tests on the observational data \cite{spirtes1991algorithm, glymour2019review}. 

\fi
\section{Experiments and Results}
\label{sec:exp}

We present three experiments: two targeting user classes (\textcolor{blue}{\textit{causal}} or \textcolor{red}{\textit{anti-causal}}) and using real data, and one targeting user subclasses (\textcolor{teal}{\textit{believers}} and \textcolor{brown}{\textit{skeptics}}) and using synthetic data.
Appendix \ref{sec:app_data} includes further details on model architectures, training procedures, and data generation.
\vspace{-3mm}

\subsection{Learning with \textit{causal} users: text-based beer recommendation} \label{sec:recbeer}


\textbf{Data.} We use \textit{RateBeer}, a dataset of beer reviews with over 3M entries and spanning $\sim 10$ years \cite{mcauley2012learning}.
We use the data to generate beer features $x$ (e.g., popularity, average rating) and $r$ (e.g., textual review embeddings) and user features $u$ (e.g., average rating, word counts). Given a sample $(u,x,r)$, our goal is to predict a (binarized) rating $y$. Here we focus on \emph{causal} users, and so would like labels $y$ to expresses causal user beliefs. The challenge is that our observational data is not necessarily such. To simulate causal user behavior, we rely on the observation that $x,r\textcolor{blue}{\rightarrow}y$ means ``changes in $x,r$ affect $y$'', and for each $u$ create an individualized empirical distribution of `counterfactual' samples $(x',r',y')$ that approximate the entire intervention space (i.e., all counterfactual outcomes $y'$ under possible interventions $(x,r)\mapsto(x',r')$). Training data 
is then generated by sampling from this space.

We consider each year as an environment $\env$, with each $\env$ inducing a distribution over $(u,x,r)$. We implement spuriousness via selection:
Each $\env$ entails different fashionable `tastes' in beer, expressed as a different weighting over the possible beer types (e.g., lager, ale, porter). Labels are then made to correlate with tastes in a certain temporal pattern. This serves as a mechanism for spurious correlation.

\begin{figure}[!t]
\centering
\begin{minipage}[t]{0.45\textwidth}
\includegraphics[width=\textwidth,valign=t]{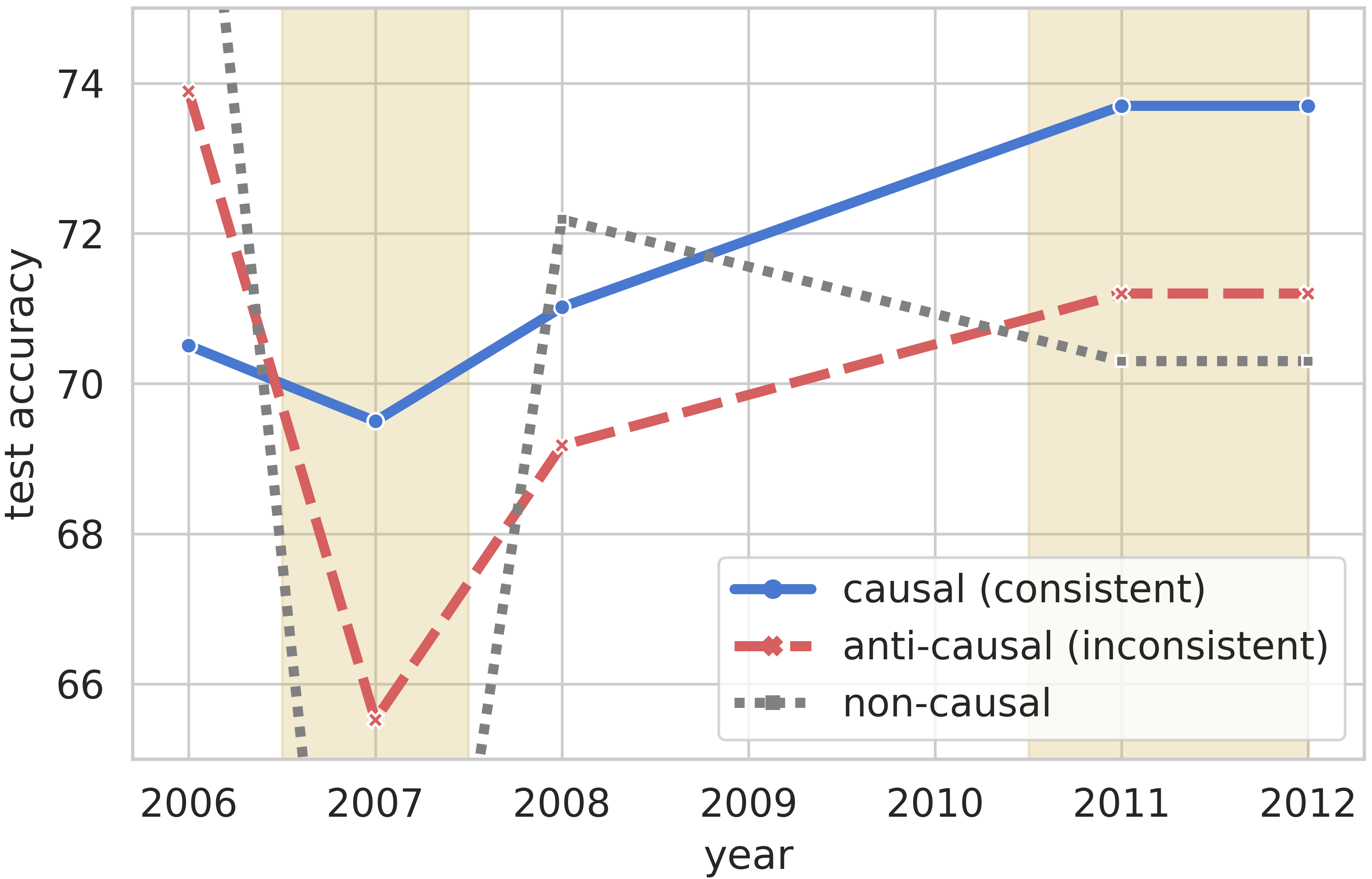}
\caption{\textbf{\textit{RecBeer} Results}. For each year, models are trained on past data (starting $2002$), and predict on the following year.
The \textcolor{blue}{\textit{causal}} training scheme, consistent with the user class, outperforms other methods when beer-type fashions ($\env$) changes.
Periods with substantial change are highlighted in \textcolor{brown}{tan}.}
\label{fig:res_beer}\end{minipage}\qquad
\begin{minipage}[t]{0.45\textwidth}
\includegraphics[width=\textwidth,valign=t]{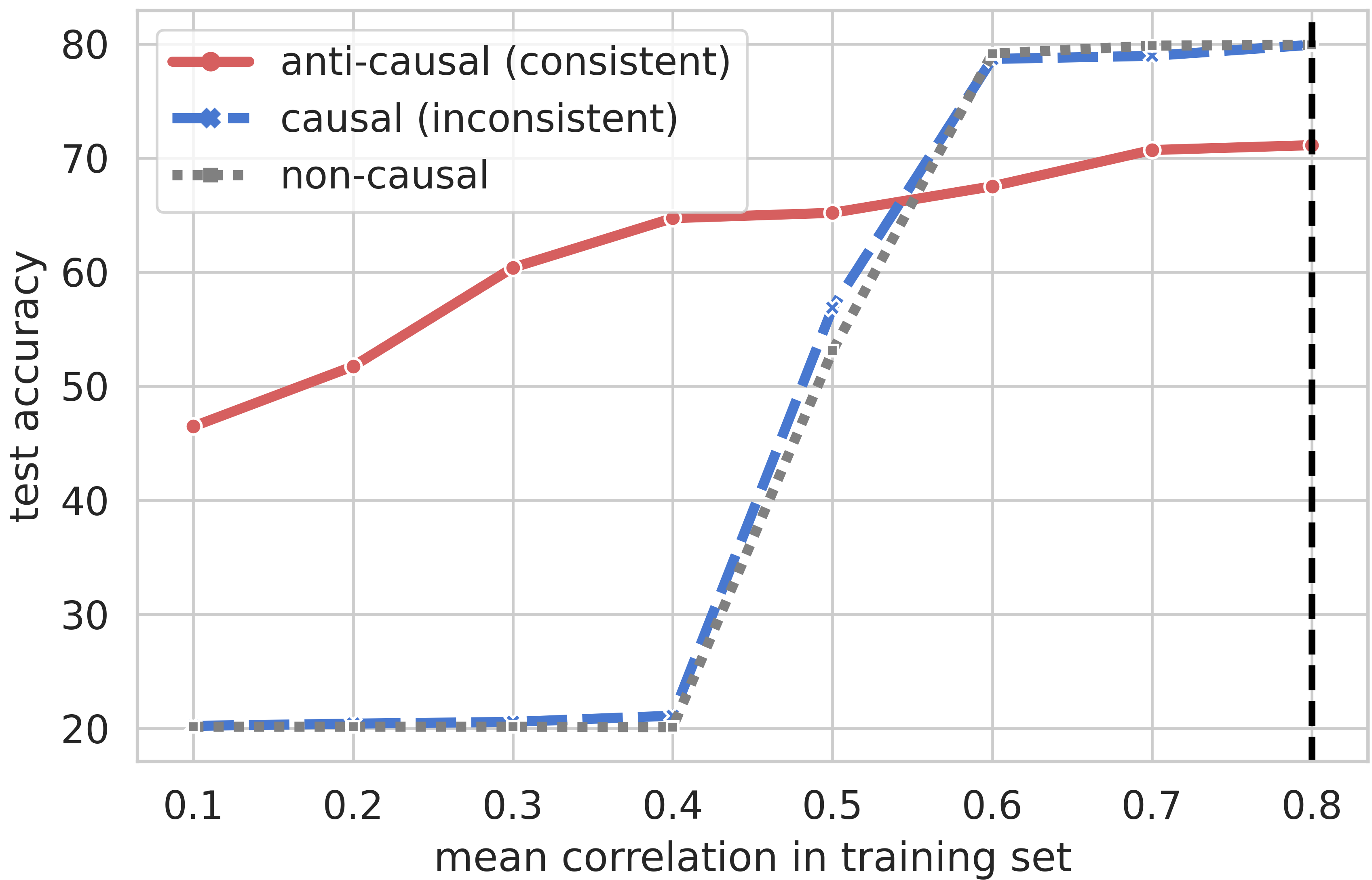}
\caption{\textbf{\textit{RecFashion} Results}.
Environments vary in the correlation between 
item colors and user choices.
The \textcolor{red}{\textit{anti-causal}} regularization scheme,
consistent with the user class, outperforms methods 
when test-time deviates from train-time correlation (=0.8).
When correlations flip ($<0.5$), other methods crash.}
\label{fig:res_fashion}\end{minipage}
\vspace{-3mm}
\end{figure}

\textbf{Results.}
Fig. \ref{fig:res_beer} compares the performance over time of three training procedures that differ only in the type of regularization applied:
\textcolor{blue}{\textit{causal}}, \textcolor{red}{\textit{anti-causal}}, and \textit{non-causal}.
Our data includes behavior generated by causal-class users;
results demonstrate the clear benefit of using a behaviorally-consistent 
regularization scheme (here, causal).
Note the causal approach is not optimal in 2006 and 2008;
this is since correlations in $\env{\leftrightarrow}y$ are set to make these years similar to the training data. However, in the face of significant shifts in taste,
other approaches collapse, while the causal approach remains stable.

\subsection{Learning with \textit{anti-causal} users: clothing-style recommendation} \label{sec:recfashion}

\textbf{Data.} We use the \textit{fashion product images} dataset\footnote{\url{https://www.kaggle.com/paramaggarwal/fashion-product-images-dataset}}, which includes includes $44.4k$ fashion items described by images, attributes, and text.
Here we focus on \emph{anti-causal} users, and generate data in a way similar to \S \ref{sec:recbeer}, but using an anti-causal intervention space.
In this experiment we let user choices $y \in \{0,1\}$ depend
on an item's image and color, which can be either red or green; in this way, $x$ is the item's grayscale image, and $r$ its hue (which we control).
Here we consider environments $\env$ that induce varying degrees of spurious correlations between color and user choices, $P(y=1|\mathtt{red})=P(y=0|\mathtt{green})=p_e$.
For the test set we use $p_e=0.8$, and experiment with training data
that gradually deviate from this relation, i.e., having $p_{e'} \in [0.1,0.8]$.


\textbf{Results.}
Fig. \ref{fig:res_fashion} shows that consistent regularization (here, anti-causal)
outperforms other alternatives whenever correlations deviate from those observed in training. Once correlations flip ($<0.5$),
both \emph{causal} and \emph{non-causal} approaches fail catastrophically;
the \emph{anti-causal} approach remains robust.

  \begin{minipage}{\textwidth}
  \begin{minipage}[t]{0.35\textwidth}
    \centering
    \begin{tikzpicture}[node distance = 1.4cm]
        \node[circle, draw, text centered] (e) {$\env$};
        \node[circle, draw, below of = e] (r) {$r$};
        \node[circle, draw, left of = r] (x_sp) {$x_{sp}$};
        \node[circle, draw, right of = r] (x_ac) {$x_{ac}$};
        \node[circle, draw, below of = r] (y) {$y$};

        \draw[->, line width = 1] (y) -- (x_ac);
        \draw[->, line width = 1] (y) -- (x_sp);
        \draw[->, line width = 1] (y) -- (r);
        \draw[teal, ->, line width = 1,transform canvas={yshift=1.5mm}] (x_sp) -- (r);
        \draw[brown, ->, line width = 1,transform canvas={yshift=-1.5mm}] (r) -- (x_sp);
        \draw[->, line width = 1] (e) -- (x_sp);
        \draw[<->, line width = 1, dash dot] (e) to[out=180, in=180, distance=2.5cm] (y);
    \end{tikzpicture}\label{graph-causal-features-2}
    
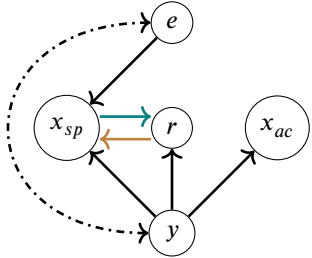
\captionof{figure}{Data-generating process for the user subclass experiment (synthetic). Here, $x$ factorizes into an \textit{anti-causal} component $x_{ac}$,
    and a spurious component $x_{sp}$ linked with $r$.
    Spuriousness results from selection bias between $y$ and $\env$.
    }
  \end{minipage}
  \hfill
  \begin{minipage}[h]{0.6\textwidth}
    \centering
    \captionof{table}{Accuracy for the user subclass experiment.
    Rows show train conditions: with and without regularization, and which users are included in the training set.
    Columns show test conditions: ID/OOD, and user type.
    Best results for each train condition (rows) are highlighted in bold.}
    \begin{tabular}{l|l|cc}  \toprule 
        \multirow{2}{*}{Reg.} & \multirow{2}{*}{Users{@}train} & \multicolumn{2}{c}{Accuracy (ID / OOD)}\\
        & & \textcolor{teal}{\textit{skeptic}} & \textcolor{brown}{\textit{believer}} \\ \midrule \midrule
        \multirow{3}{*}{$\lambda=0$} & \textcolor{teal}{\textit{skeptic}} & \textbf{78.0} / 50.0  & \textbf{89.8} / 75.1\\
        & \textcolor{brown}{\textit{believer}} & \textbf{78.0} / 50.0  & \textbf{89.8} / 75.1\\
        & both & \textbf{78.0} / 50.0  & \textbf{89.8} / 75.1\\ \midrule
        \multirow{3}{*}{$\lambda>0$} & \textcolor{teal}{\textit{skeptic}} & 71.1 / \textbf{75.6} & \gray{74.67 / 75.5} \\
        & \textcolor{brown}{\textit{believer}} & \gray{69.8 / 52.5} & 88.5 / \textbf{85.2}\\
        & both & 70.03 / 64.57 & 78.88 / 78.13  \\
        \bottomrule
    \end{tabular}\label{tab:toy_results}
    \end{minipage}
  \end{minipage}

\subsection{Learning with multiple user subclasses}
\label{subsec:synthetic}

Our final experiment studies learning with users of
of the same-class (here, \textit{anti-causal}) but different subclasses: \textcolor{teal}{\textit{skeptics}} or \textcolor{brown}{\textit{believers}}.
Our analysis in \S \ref{sec:subclasses} suggests that each user subclass
may have a different optimal predictor; here we investigate this empirically on synthetic data.





\textbf{Data.} 
The data-generating process is as follows (see Fig. \ref{graph-causal-features-2}).
We use three environments: $e_1,e_2$ at train, and $e_3$ at test,
and implement a selection mechanism (dashed line) that causes differences in $p(y|\env_i)$
across $\env_i$.
Since we focus on anti-causal users, features $x,r$ are determined by $\env,y$.
We use three binary features: $x_{sp}$ (`spurious'), $x_{ac}$ ('anti-causal'), and $r$.
These are designed so that an $f$ which uses $x_{ac}$ alone obtains $0.75$ accuracy,
but using also $x_{sp}$ improves \emph{in-distribution} (ID) accuracy slightly to $0.78$,
and so the optimal ID predictor for both user subclasses is of the form $f^*(x_{ac},x_{sp})$.
However, relying on $x_{sp}$ causes \emph{out-of-distribution} (OOD) performance to deteriorate considerably;
thus, robust models should not learn to discard $x_{sp}$.
The role of $r$ is to distinguish between user subclasses:
The \textcolor{teal}{skeptic} does not need $r$ since, for her, it is fully determined by $x_{sp}$;
the optimal invariant predictor is hence $\fbel(x_{ac})$.
Meanwhile, the \textcolor{brown}{believer}, due to the v-structure
$r \textcolor{brown}{\rightarrow} x_{sp} {\leftarrow} \env$,
can benefit in-distribution by using both $r$ and $x_{sp}$;
here, the optimal invariant predictor is $\fskep(x_{ac},r)$.

\textbf{Results.}
Table \ref{tab:toy_results} shows ID and OOD performance for each user subclass (columns), for learning with and without regularization (rows).
Since all users are anti-causal, we use anti-causal (i.e., conditional) regularization.
We compare learning a separate predictor for each user type (rows `\emph{skeptic}' and `\emph{believer}') and learning a single predictor over all users jointly (`both').
Results show that without regularization ($\lambda=0$), ID performance is good, but the learned predictor fails OOD---drastically for skeptic users (note all rows are the same since
both user types share the same ID-optimal $f^*$).
In contrast, when regularization is applied ($\lambda>0$),
learning an independent predictor for each user subclass
performs well OOD (for both subclasses), indicating robustness to changing environments;
note that ID performance is also mostly maintained.
Meanwhile, learning on the entire dataset (i.e., including both user types)
does provide some robustness---but is suboptimal both ID and OOD.

\subsection{Learning with mixed sub-populations}
Our previous experiment considered a setting in which the learner has exact information regarding each user's sub-type, and so can correctly partition the population in a way that is optimal in regards to Prop. \ref{prop:subclass}.
However, such precise knowledge may not be available in practice,
or may be too costly or difficult to infer.
In this section we experiment in a setting where the learner has only coarse information (or general beliefs) about user (sub-)types.
Our results suggests that following the practical conclusions of Prop. \ref{prop:subclass}---namely partitioning the population of users based on (estimated) types and learning a different predictive model for each---can be beneficial even when based only on a reasonable guess.

\begin{wrapfigure}{r}{0.47\textwidth}
    \centering
    \includegraphics[width=\linewidth]{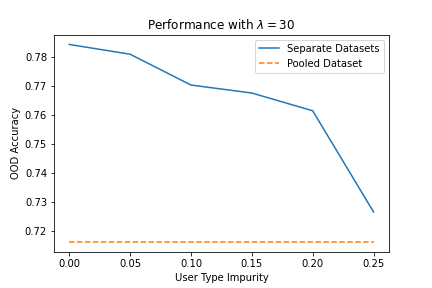}
    \caption{Learning with a mixed population of believers and skeptics. Even when the minority group is large (25\%),
    learning in a way that is tailored to the majority group is still
    beneficial.
    }
    \vspace{-3mm}
    \label{fig:believers_skeptics_stability}
\end{wrapfigure}

\textbf{Data.}
We use the setting of Sec.\ref{subsec:synthetic}
with a population of anti-causal users, composed of two sub-populations of \textcolor{teal}{skeptics} and \textcolor{brown}{believers}.
We then simulate a setting where the learner has imprecise information about user types by adding noise:
we move an $\alpha$-fraction of each subpopulation into the other,
this creating two mixed sub-populations.
with increasing levels of `impurity', $\alpha \in [0,0.25]$.
Here, $\alpha=0$ represents perfect information,
whereas $\alpha=0.25$ entails large minority groups (25\%). 


\textbf{Results.} Fig. \ref{fig:believers_skeptics_stability}
compares OOD performance of the robust model (solid line)
to a \naive\ model trained on the pooled dataset (dashed line).
In line with previous results, for $\alpha=0$, the robust model 
achieves significantly higher accuracy on the test environment.
As $\alpha$ increases, the robust model preserves its advantage, with accuracy degrading gracefully;
for $\alpha=0.25$, the robust model still outperforms the pooled baseline.
Thus, our results show that despite having imperfect information regarding user sub-types,
learning distinct models for each subpopulation,
as Prop. \ref{prop:subclass} suggests,
remains beneficial.

\vspace{-2mm}
\section{Discussion}
\label{sec:disc}
\vspace{-2mm}

Humans beings perceive the world causally; our paper argues that to cope with a world that \emph{changes},
learning must take into account how humans believe these changes take effect.
We identify one key reason: in making decisions under uncertainty,
users can \emph{cause} spurious correlations to appear in the data.
Towards this, we propose to employ tools from invariant causal learning, but in a way that is tailored to how humans make decisions, this drawing on economic models of bounded-rationality.
Our approach relies on regularization for achieving invariance, with our main point being that \emph{how} and \emph{what} to regularize can be derived from users' causal graphs.
Although we have argued that even partial graph information can be helpful---even this form of knowledge is not straightforward to obtain (notably at test-time), and may require experimentation.
Nonetheless, and in hopes of spurring further interest,
we view our work as taking one step towards establishing a disciplined
causal perspective on the interaction between recommending systems and the decision-making users they aim to serve.

\section*{Acknowledgements}
This work was supported in part by The Israel Science Foundation (grant 278/22).

\bibliography{recause}

\newpage
\appendix



\section{Details on Formal Claims} \label{sec:app_theory}
Our claim in Proposition~\ref{prop:class} is also based on the setting of \citet{veitch2021counterfactual}. Under the assumption that $e$ is discrete, Lemma~3.1 of \citep{veitch2021counterfactual} ensures that there exists a random variable $(x,r)^{\perp}_e$ such that $f_u(x,r)$ is CI if and only if it is $(x,r)^{\perp}_e$-measurable. Then we will assume that $x,r$ can be decomposed into parts $x,r_{y\wedge e}, x,r^{\perp}_{y}, x,r^{\perp}_{e}$. Note that we do not assume that we know how to decompose our features in this manner, nor we assume anything about the semantic meaning of these components. We only assume that this decomposition exists, and then the main assumption made in \citep{veitch2021counterfactual} is that the graph in Fig.~\ref{graph-causal}a conforms to the structures in Figure~\ref{user-types-expanded} for each user type.
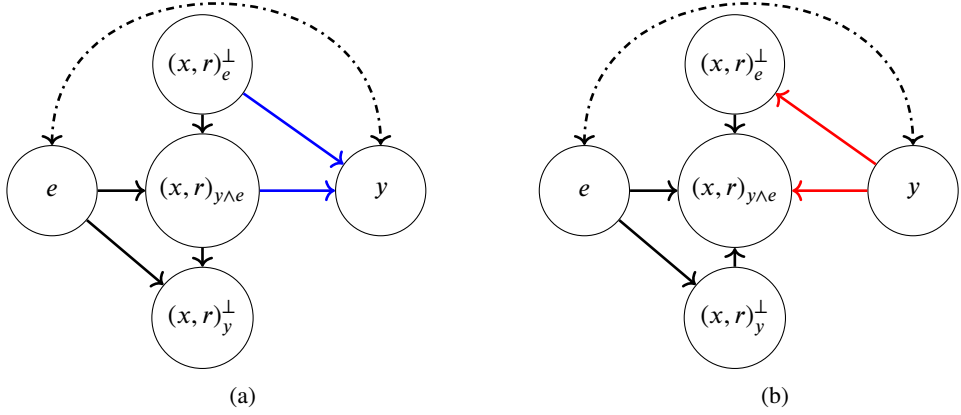
\begin{figure}[ht]
\captionsetup[subfigure]{labelformat=empty}
    \centering
    \begin{subfigure}{0.45\textwidth}
    \begin{tikzpicture}[node distance = 2cm]
        \node[circle, draw, text centered, minimum size=1.2cm] (x) {$(x,r)_{y\wedge e}$};
        \node[circle, draw, text centered, minimum size=1.2cm, above=0.25cm of x] (x_caus) {$(x,r)_{e}^{\perp}$};
        \node[circle, draw, text centered, minimum size=1.2cm, below=0.25cm of x] (x_spu) {$(x,r)_{y}^{\perp}$};
        \node[circle, draw, minimum size=1.2cm, text centered, right=1cm of x] (y) {$y$};
        \node[circle, draw, left of = x, minimum size=1.2cm, text centered] (e) {$\env$};
        
        \draw[blue, ->, line width = 1] (x) -- (y);
        \draw[blue, ->, line width = 1] (x_caus) -- (y);
        \draw[->, line width = 1] (e) -- (x);
        \draw[->, line width = 1] (e) -- (x_spu);
        \draw[->, line width = 1] (x_caus) -- (x);
        \draw[->, line width = 1] (x) -- (x_spu);
        \draw[<->, line width = 1, dash dot] (e) to[out=90, in=90, distance=2.5cm] (y);
        \end{tikzpicture}
    \caption{(a)}\label{graph-causal-expanded}
    \end{subfigure}
    \qquad
    \begin{subfigure}{0.45\textwidth}
    \begin{tikzpicture}[node distance = 2cm]
        \node[circle, draw, text centered, minimum size=1.2cm] (x) {$(x,r)_{y\wedge e}$};
        \node[circle, draw, text centered, minimum size=1.2cm, above=0.25cm of x] (x_caus) {$(x,r)_{e}^{\perp}$};
        \node[circle, draw, text centered, minimum size=1.2cm, below=0.25cm of x] (x_spu) {$(x,r)_{y}^{\perp}$};
        \node[circle, draw, minimum size=1.2cm, text centered, right=1cm of x] (y) {$y$};
        \node[circle, draw, left of = x, minimum size=1.2cm, text centered] (e) {$\env$};
        
        \draw[red, ->, line width = 1] (y) -- (x);
        \draw[red, ->, line width = 1] (y) -- (x_caus);
        \draw[->, line width = 1] (e) -- (x);
        \draw[->, line width = 1] (e) -- (x_spu);
        \draw[->, line width = 1] (x_spu) -- (x);
        \draw[->, line width = 1] (x_caus) -- (x);
        \draw[<->, line width = 1, dash dot] (e) to[out=90, in=90, distance=2.5cm] (y);
    \end{tikzpicture}
    \caption{(b)}\label{graph-anti-causal-expanded}
    \end{subfigure}
    \caption{
    Detailed graphs describing our assumptions on causal and anti-causal users
    (a) causal model for data generating process of \textcolor{blue}{\textit{causal}} user, and
    (b) \textcolor{red}{\textit{anti-causal}} user.
    Dashed lines indicate possible confounding.
    }
    \label{user-types-expanded}
    \vspace{-2mm}
\end{figure}

We are now ready to state Proposition~\ref{prop:class} in a more precise manner 

\begin{proposition}
Let $f$ be a CI model and assume $y$ and $\env$ are confounded (i.e. they are connected by an unobserved common cause $c$ or by a directed path). Further assume that $D^e(x,r,y \mid u)$ is entailed by the causal models in Fig.~\ref{user-types-expanded} for $u=\uc$ and $u=\uac$.
Then the following holds:
    \setlist{nolistsep}
    \begin{enumerate}[noitemsep,leftmargin=0.5cm]
    \item $f_{\uc}$ must satisfy $D^e(f_{\uc}(x,r)) = D^{e'}(f_{\uc}(x,r))\,\,\, \forall e,e'\in{\envs}$.
    \item $f_{\uac}$ must satisfy $D^e(f_{\uac}(x,r) \mid y) = D^{e'}(f_{\uac}(x,r) \mid y) \,\,\, \forall e,e'\in{\envs}$, $y\in{\{0, 1\}}$.
    \end{enumerate}
On the other hand, $f_\uac$ and $f_\uc$ do not necessarily satisfy conditions 1 and 2, respectively. 
\end{proposition}
\begin{proof}
Under the assumptions laid out about the causal model, the conditional independence relations can be read off the graph directly, as in Theorem~3.2 of \citep{veitch2021counterfactual}. This proves that the independence properties stated in the proposition must hold.
To see that $f_{\uac}, f_{\uc}$ do not necessarily satisfy properties $1$ and $2$ respectively, we will prove the existence of such cases. Consider a causal model where $e$ and $y$ are confounded, and assume that the model is faithful \citep{pearl2009causality} (i.e. all conditional independence statements that are not entailed by the graph do not hold). Hence for the causal user we generally have $D^e((x,r)_e^\perp \mid y, u=\uc) \neq D^{e'}((x,r)_e^\perp \mid y, u=\uc)$ (at the very least there are values  of $(x,r)^\perp_{e}, y$ for which this holds), and hence there exists some $(x,r)_e^\perp$-measurable function $\hat{f}_{\uc}(x,r)$ that satisfies $D^e(\hat{f(x,r)} \mid y, u=\uc) \neq D^{e'}(\hat{f(x,r)} \mid y, u=\uc)$. The same argument can be applied for the anti-causal user $\uac$ to prove the existence of an $(x,r)_e^\perp$-measurable function $\hat{f}_{\uac}(x,r)$ that satisfies $D^e(\hat{f(x,r)} \mid u=\uac) \neq D^{e'}(\hat{f(x,r)} \mid u=\uac)$. The model $\hat{f}(x,r)$ is CI since the constructed functions are $(x,r)_e^\perp$-measurable, but models $f_{\uac}, f_{\uc}$ do not satisfy conditions $1$ and $2$ respectively, which concludes our claim.
\end{proof}
Next we prove Proposition~\ref{prop:subclass} by constructing a confounded model for an anti-causal user, similar to the one in the synthetic experiment of Section~\ref{subsec:synthetic}. Towards this proposition, we point out that an optimal CI predictor is defined as a CI predictor with the best possible worst case performance. Where the worst case is taken over all distributions that are causally-compatible \cite{veitch2021counterfactual} with the source distribution $D_{\text{train}}$.
\begin{definition}
$D_{\text{train}}$ and $D_{\text{OOD}}$ are causally compatible if they are entailed by the same causal graph, $D_{\text{train}}(y)=D_{\text{OOD}}(y)$, and there is a confounder $c$ and/or selection conditions $s, \tilde{s}$ such that $D_{\text{train}}=\int D_{\text{train}}(x_{sp}, x_{ac}, r, y \mid c, s=1)d\tilde{P}(c)$ and $D_{\text{OOD}}=\int D_{\text{train}}(x_{sp}, x_{ac}, r, y \mid c, \tilde{s}=1)d\tilde{Q}(c)$ for some $\tilde{P}(c), \tilde{Q}(c)$.
\end{definition}
Let us focus now on distributions where $f(x_{sp}, r, x_{ac})$ is counterfactually invariant if and only if it is $(r,x_{ac})$-measurable (the expression $(r,x_{ac})$ should be read as a bivariate random variable). Note again that from Lemma~3.1 of \cite{veitch2021counterfactual} such a variable exists. The following claim will help us reason about the optimal CI model for users of the skeptic sub-class.
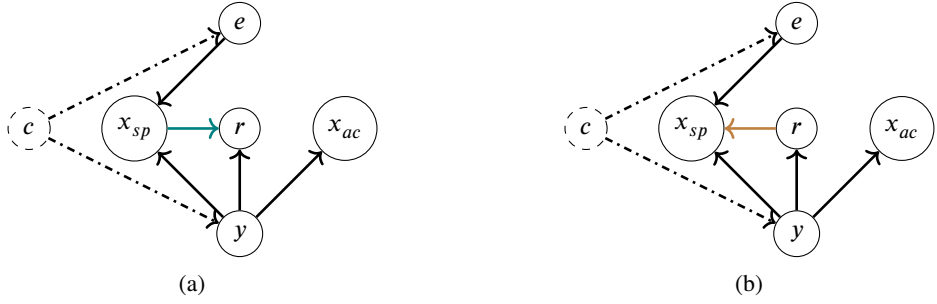
\begin{figure}
\begin{subfigure}[t]{0.47\textwidth}
    \centering
    \begin{tikzpicture}[node distance = 1.4cm]
        \node[circle, draw, text centered] (e) {$\env$};
        \node[circle, draw, below of = e] (r) {$r$};
        \node[circle, draw, left of = r] (x_sp) {$x_{sp}$};
        \node[circle, draw, dashed, left of = x_sp] (c) {$c$};
        \node[circle, draw, right of = r] (x_ac) {$x_{ac}$};
        \node[circle, draw, below of = r] (y) {$y$};

        \draw[->, line width = 1] (y) -- (x_ac);
        \draw[->, line width = 1] (y) -- (x_sp);
        \draw[->, line width = 1] (y) -- (r);
        \draw[teal, ->, line width = 1] (x_sp) -- (r);
        \draw[->, line width = 1] (e) -- (x_sp);
        \draw[->, line width = 1, dash dot] (c) to (y);
        \draw[->, line width = 1, dash dot] (c) to (e);
    \end{tikzpicture}
    \caption{}\label{graph-skeptic-expanded}
  \end{subfigure}
  \hfill
    \begin{subfigure}[t]{0.47\textwidth}
    \centering
    \begin{tikzpicture}[node distance = 1.4cm]
        \node[circle, draw, text centered] (e) {$\env$};
        \node[circle, draw, below of = e] (r) {$r$};
        \node[circle, draw, left of = r] (x_sp) {$x_{sp}$};
        \node[circle, draw, dashed, left of = x_sp] (c) {$c$};
        \node[circle, draw, right of = r] (x_ac) {$x_{ac}$};
        \node[circle, draw, below of = r] (y) {$y$};

        \draw[->, line width = 1] (y) -- (x_ac);
        \draw[->, line width = 1] (y) -- (x_sp);
        \draw[->, line width = 1] (y) -- (r);
        \draw[brown, ->, line width = 1] (r) -- (x_sp);
        \draw[->, line width = 1] (e) -- (x_sp);
        \draw[->, line width = 1, dash dot] (c) to (y);
        \draw[->, line width = 1, dash dot] (c) to (e);
    \end{tikzpicture}
    \caption{}\label{graph-believer-expanded}
    \end{subfigure}
    \caption{Graphs describing the data-generating processes for anti-causal \emph{\textcolor{teal}{believer}} and \emph{\textcolor{brown}{skeptic}} users in the proof of Proposition~\ref{prop:subclass}.
    }
\end{figure}
\begin{lemma}
If $D_{\text{train}}$ is entailed by the graph in Fig.~\ref{graph-skeptic-expanded} and $D_{\text{OOD}}$ is causally compatible with it, then $D_{\text{train}}(y \mid r, x_{ac}) = D_{\text{OOD}}(y \mid r, x_{ac})$.
\end{lemma}
\begin{proof}
For binary classification, it is enough to show that $\frac{D_{\text{train}}(y=1 \mid r, x_{ac})}{D_{\text{train}}(y=0 \mid r, x_{ac})} = \frac{D_{\text{OOD}}(y=1 \mid r, x_{ac})}{D_{\text{OOD}}(y=0 \mid r, x_{ac})}$. Let us write this for the training distribution:
\begin{align*}
    \frac{D_{\text{train}}(y=1 \mid r, x_{ac})}{D_{\text{train}}(y=0 \mid r, x_{ac})} &= \frac{D_{\text{train}}(r, x_{ac} \mid y=1)D_{\text{train}}(y=1)}{D_{\text{train}}(r, x_{ac} \mid y=0)D_{\text{train}}(y=0)} \\
    &= \frac{D_{\text{train}}(r, x_{ac} \mid y=1)D_{\text{OOD}}(y=1)}{D_{\text{train}}(r, x_{ac} \mid y=0)D_{\text{OOD}}(y=0)}.
\end{align*}
The second equality stems from the causal-compatibility of $D_{\text{OOD}}$. It is left to show that $D_{\text{train}}(y \mid r, x_{ac}) = D_{\text{OOD}}(y \mid r, x_{ac})$. From causal-compatibility the distributions are entailed by the same graph in Fig.~\ref{graph-skeptic-expanded}, which imposes the conditional independence $c \indep r, x_{ac} \mid y $. Hence we conclude the proof by:
\begin{align*}
    D_{\text{train}}(x_{ac}, r \mid y) = \int D_{\text{train}}(x_{ac}, r \mid y, c)d\tilde{P}(c) = \int D_{\text{train}}(x_{ac}, r \mid y, c)d\tilde{Q}(c) = D_{\text{OOD}}(x_{ac}, r \mid y).
\end{align*}
\end{proof}
From this result we gather that if we only consider the features $x_{ac},r$, there is a unique Bayes-optimal classifier over all target distributions that are causally compatible with $D_{\text{train}}$. Since a classifier is CI if and only if it is $(x_{ac},r)$-measurable, we see that for the skeptic sub-class of users the optimal CI model is $f(x_{sp}, r, x_{ac}) = D_{\text{train}}(y \mid r, x_{ac})$. The rest of the proof will simply show that this model may not be CI for a user of sub-type \emph{\textcolor{teal}believer} that has the same choice patterns over observed data pooled from two training environments.
\begin{proof}[Proof of Proposition~\ref{prop:subclass}]
Consider a data generating process as depicted in Figure~\ref{graph-skeptic-expanded}. All variables $x_{sp},r,x_{ac},y,c$ are binary, we consider $2$ training environments $\envs_{\text{train}}=\{0, 1\}$.
We write down the distribution in a factorized form:
\begin{align*}
    D_{\uskep}(x_{sp},x_{ac},r,y) &= \sum_{c\in{\{0,1\}}, e\in{\{0, 1\}}}p(c)p(y \mid c)p(x_{ac} \mid y)p(e \mid c)p_{\uskep}(r \mid y)p_{\uskep}^e(x_{sp} \mid r, y) \\
    &= p(x_{ac} \mid y)p_{\uskep}(r \mid y)\left(\sum_{e\in{\{0, 1\}}}{\tilde{p}(e, y)p^e_{\uskep}(x_{sp}\mid r, y)}\right).
\end{align*}
Here we defined $\tilde{p}(e, y) = \sum_{c\in{0, 1}}{p(y, c) p(e \mid c)}$. The subscripts $\uskep$ emphasize that in the distribution we will construct for the believer user, $D_{\ubel}$, all factors that are not subscripted will be equal to those in $D_{\uskep}$. That is, consider a distribution that factorizes over the graph in Figure~\ref{graph-believer-expanded} as follows:
\begin{align} \label{eq:bel_factor}
    D_{\ubel}(x_{sp}, x_{ac}, r, y) = p(x_{ac} \mid y)p_{\ubel}(r \mid y, x_{sp})\left(\sum_{e\in{\{0, 1\}}}{\tilde{p}(e, y)p^e_{\ubel}(x_{sp} \mid y)}\right).\qquad\qquad
\end{align}
We will show that there exists some setting of $p_{\ubel}(r \mid y, x_{ac}), p^e_{\ubel}(x_{sp} \mid y)$ such that:
\begin{align*}
D_{\uskep}(x_{sp}, x_{ac}, r, y) = D_{\ubel}(x_{sp}, x_{ac}, r, y).
\end{align*}
But it will also satisfy $D^0_{\ubel}(y \mid r, x_{ac}) \neq D^1_{\ubel}(y \mid r, x_{ac})$.
Then the proof will be concluded, as $f(x_{sp}, x_{ac}, r) = D_{\uskep}(y \mid r, x_{ac}) = D_{\ubel}(y \mid r, x_{ac})$ cannot be CI w.r.t $D_{\ubel}$.
This holds since $D_{\ubel}^e(y \mid r, x_{ac}) \neq D_{\ubel}(y \mid r, x_{ac})$ for $e\in{\{0, 1\}}$, hence there must be some instance for which $f(x_{ac}(0), x_{sp}(0), r(0)) \neq f(x_{ac}(1), x_{sp}(1), r(1))$.

Towards this, consider $D_{\uskep}(r \mid y, x_{sp})$ which is obtained by the respective marginalization and conditioning of $D_{\uskep}(x_{sp}, x_{ac}, r, y)$, and also consider $\sum_{e\in{0,1}}{\tilde{p}(e, y)D^e_{\uskep}(x_{sp}\mid y)}$. Let us set:
\begin{align*}
    p_{\ubel}(r \mid y, x_{sp}) := D_{\uskep}(r \mid y, x_{sp}).
\end{align*}
It is clear that if we set $p^e_{\ubel}(x_{sp} \mid y)$ such that the following holds:
\begin{align} \label{eq:required_eq}
    \sum_{e\in{\{0, 1\}}}{\tilde{p}(e, y)p^e_{\ubel}(x_{sp} \mid y)} = \sum_{e\in{\{0, 1\}}}{\tilde{p}(e, y)D^e_{\uskep}(x_{sp} \mid y)},
\end{align}
then the equality $D_{\uskep}(x_{sp},x_{ac},r,y) = D_{\ubel}(x_{sp},x_{ac},r,y)$ also holds. That is because the factorization in \eqref{eq:bel_factor} is a factorization of the joint distribution over $x_{sp}, x_{ac}, r, y$ where all factors are equal to the ones obtained from $D_{\uskep}(x_{sp}, x_{ac}, r, y)$. \footnote{Note that it is easy to observe that the two sides of \eqref{eq:required_eq} are the marginal distribution over $x_{sp}, y$ of the two distributions $D_{\ubel}$ and $D_{\uskep}$ respectively.}

Finally, we claim that many solutions satisfy \eqref{eq:required_eq}. For each value of $y, x_{sp}$  Eq.~\eqref{eq:required_eq} is a linear equation with two variables ($p^0_{\ubel}(x_{sp} \mid y)$ and $p^1_{\ubel}(x_{sp} \mid y)$), and they should be constrained to take values in the range $[0,1]$.
One solution to the equation is to set $p^e_{\ubel}(x_{sp} \mid y):=D^e_{\uskep}(x_{sp} \mid y)$, and unless $D^e_{\uskep}(x_{sp} \mid y)\in{\{0, 1\}}$
for each value of $x_{sp}, y$, and $D^0_{\uskep}(x_{sp} \mid y) = D^1_{\uskep}(x_{sp} \mid y)$ (i.e. the spurious feature completely determines $y$) the set of solutions to the equations forms an interval in $\mathbb{R}^2$, and has Lebesgue measure that is non-zero.

Thus let us consider the set of parameterized (by the factors in \eqref{eq:bel_factor}) distributions $\tilde{D}_{\ubel}(e,x_{sp},r,x_{ac},y)$ that satisfy $\sum_{\tilde{e}}{\tilde{D}_{\ubel}(e=\tilde{e},x_{sp},r,x_{ac},y)} = D_{\uskep}(x_{sp},r,x_{ac},y)$ for the fixed distribution $D_{\uskep}(x_{sp},r,x_{ac},y)$. This set has a non-zero Lebesgue measure over the linearly independent parameters needed to parameterize $D_{\ubel}$. Since the set of parameters that yield unfaithful distributions w.r.t a graph has Lebesgue measure zero \cite{spirtes2000causation}, there must be at least one distribution $\tilde{D}_{\ubel}(e,x_{sp},r,x_{ac},y)$ in the set where the independence $r,x_{ac} \indep e \mid y$ does \emph{not} hold. For such a distribution we will have $D^e_{\ubel}(y \mid r, x_{ac}) \neq D_{\ubel}(y \mid r, x_{ac})$, which is what was required to conclude the proof.

\end{proof}

\section{Experimental Details}
\label{sec:app_data}

Code and data for all experiments can be found in the following anonymous link:\\
\url{https://drive.google.com/drive/folders/1bO57v4PUuUh76F_q0a_xAVx6CKdeDJ5l}


\subsection{\emph{RecBeer} (causal users)}

\paragraph{Original Dataset description.}
The original \emph{RateBeer} dataset includes textual reviews and numerical ratings of roughly 3000 unique beers, collected over the span of over 11 years.
Each review data-point also includes additional features describing the beer (e.g., brand, style), the author of the review (e.g., location), and the review itself (e.g., date).
Figure \ref{fig:rate_beer} shows an example of a data point.
Table \ref{tab:ratebeer_stats} provides summary statistics.


\begin{table}[t!]
\centering
\small
 \begin{minipage}[t]{0.45\linewidth}
    \centering
    \caption{Original \textit{RateBeer} dataset statistics.}
    \begin{tabular}{l|l} \toprule 
        Number of reviews & $2,924,127$ \\
        Number of users & $40,213$ \\
        Number of beers & $110,419$ \\
        Users with > 50 reviews & $4,798$ \\
        Median \#words per review & $54$ \\
        Timespan & 4/2000-11/2011 \\ \bottomrule
    \end{tabular}
    \label{tab:ratebeer_stats}
\end{minipage}
\quad
\begin{minipage}[t]{0.45\linewidth}
    \centering
    \caption{Our \textit{RecBeer} data features.}
    \begin{tabular}{l|c|l}  \toprule 
        Variable type & Not.
        & Description \\ \midrule \midrule
        \multirow{6}{*}{Item} & \multirow{6}{*}{$x$} & avg past appearance \\
        & & avg past aroma \\
        & & avg past palate \\
        & & avg past taste \\ 
        & & \# of active years \\ 
        & & alcohol percentage \\ 
        & & beer type \\ \midrule
        \multirow{3}{*}{User} & \multirow{3}{*}{$u$} & avg past satisfaction \\
        & & \# of past choices \\
        & & \# of active years \\ \midrule
        \multirow{2}{*}{Recommendation} & \multirow{2}{*}{$r$} & text review \\ 
        & & \# of past reviews \\ \midrule
        Time & $\env$ & year \\  \midrule
        Choice & $y$ & try beer/not \\ \bottomrule
    \end{tabular}
    
    \label{tab:recbeer_features}
    \end{minipage}
\end{table}

\begin{figure}[p!]
    \centering
    \includegraphics[width=0.8\linewidth]{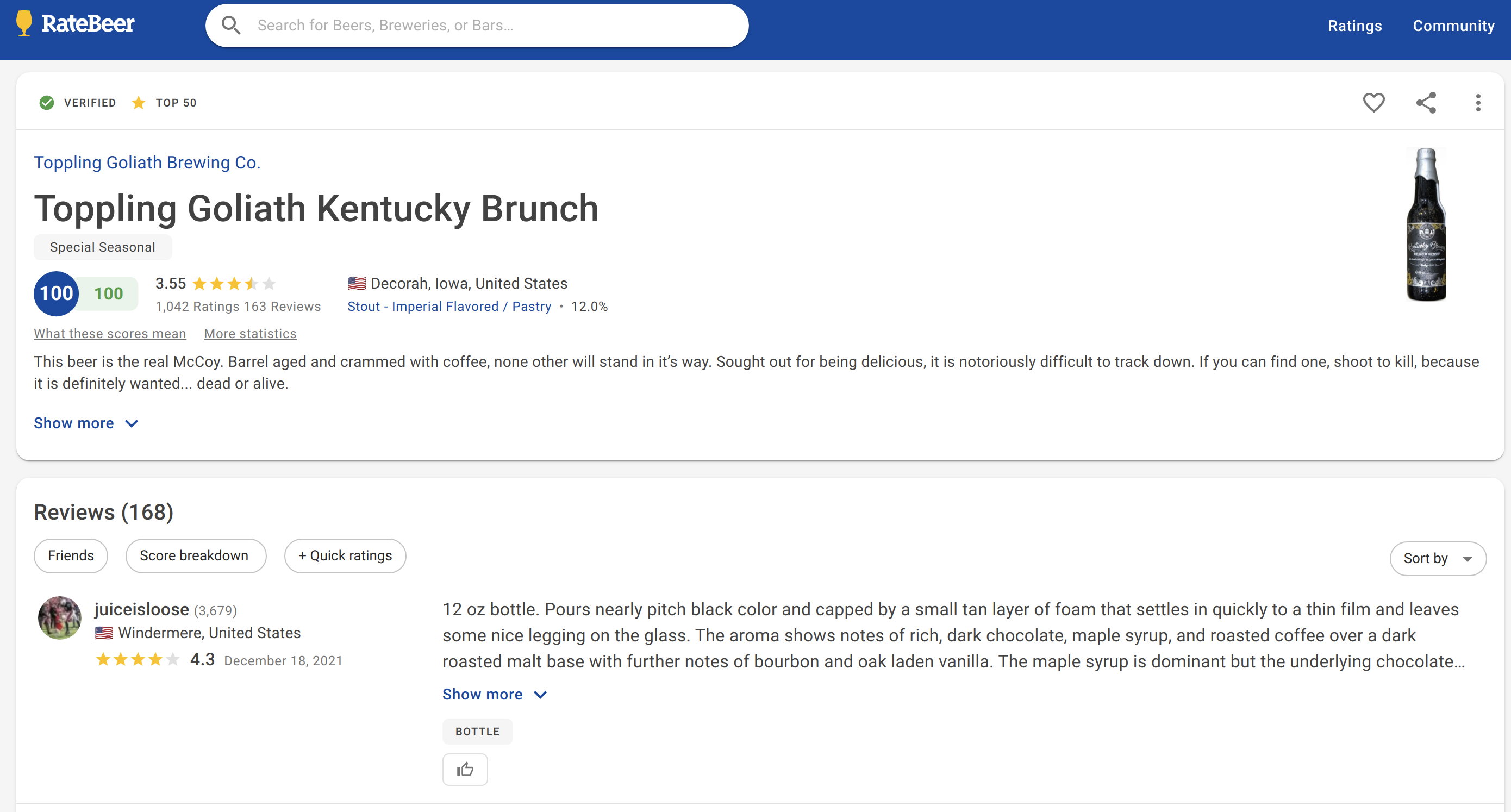}
    \caption{\textbf{\emph{RateBeer} example:} A textual review and numerical rating for a beer (with metadata).}
    \label{fig:rate_beer}
\end{figure}

\begin{figure}[p!]
    \centering
    \includegraphics[width=0.8\linewidth]{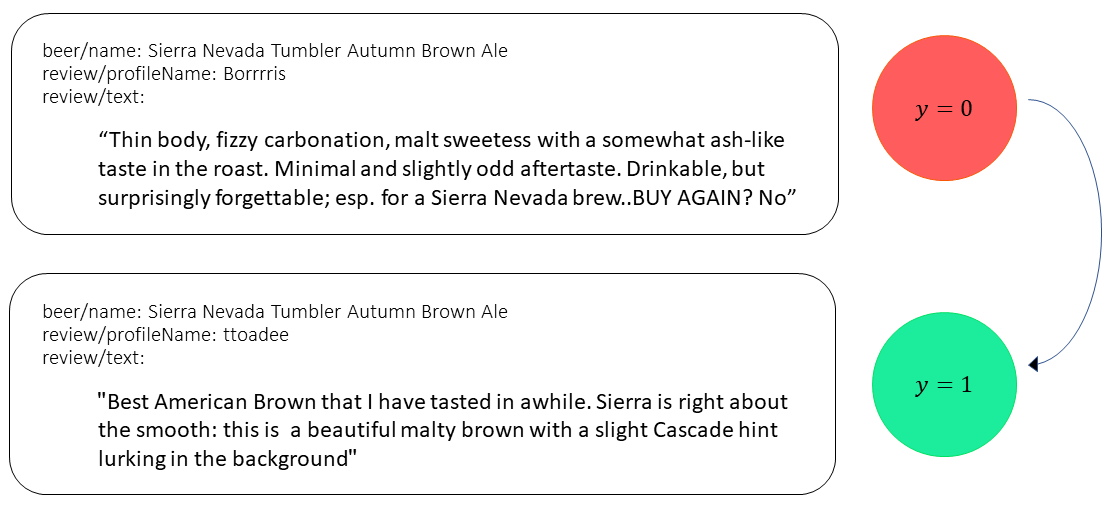}
    \caption{\textbf{\textit{RecBeer} interventions:}
    An example of a simulated intervention for causal users,
    for which changing the review shown to the user (bottom)
    to another (top) may influence his behavior (here, from not choosing to choosing).
    }
    \label{fig:reviews_swap}
\end{figure}

\begin{figure}[p!]
    \centering
    \includegraphics[width=0.65\linewidth]{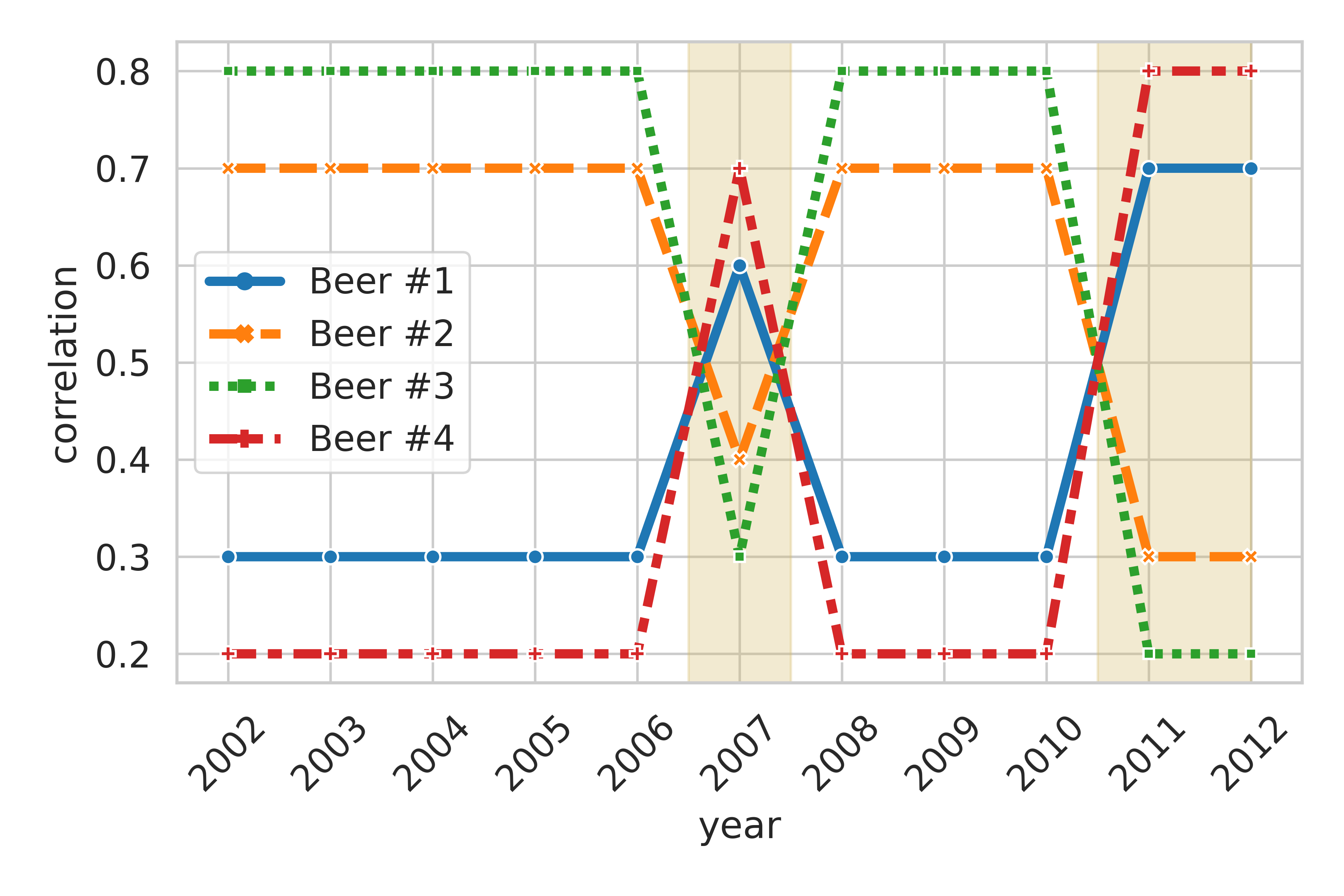}
    \caption{\textbf{\emph{RecBeer} environments:}
    Each year serves as a different environment, whose affect is expressed through differing correlations between beer types and user choices.
    The plot shows the temporal correlation structure used for the experiment in \S \ref{sec:recbeer}, and underlie the results presented in Fig. \ref{fig:res_beer}.
    Periods with substantial changes are highlighted in \textcolor{brown}{tan}.}
    \label{fig:corrs_beer}
\end{figure}

\paragraph{Data Generation Process.}
The original RateBeer dataset includes 
reviews and rating that were authored and submitted by users of the platform.
For our purposes, focusing learning and prediction on users as \emph{contributors} of content has two limitations:
(i) we cannot know what platform-selected information ($r$) was presented to them and how it influenced their decisions,
and (ii) we cannot reason counterfactually about their potential choices had they been exposed to different information.

To overcome both issues, we adapt the original dataset to simulate choice behavior of users as \emph{consumers} of content,
as they use the platform to make informed decisions about beer consumption.
We emulate the following process:
a user $u$ logs on to the platforms, and is recommended a certain beer.
The beer is described by intrinsic features $x$,
and one platform-selected textual review $r$,
chosen from a pool of already-existing reviews for that beer (these being the reviews for that beer that have already by submitted by other contributing users).
The user then decides weather to try (i.e., consume) the beer ($y=1$) or not ($y=0$).
Our goal is to predict for new users $u$ their choices $y$ for recommended beers given descriptions $x,r$.

To create features for beers $x$ and (consuming) users $u$,
we aggregate information from all corresponding reviews:
for beers---all reviews of that beer, and for users---all reviews authored by that user.
This includes features such as average past taste score for beers and average past overall satisfaction for users. Table \ref{tab:recbeer_features} summarizes our feature space. 
Since we model users as \emph{causal}, the graph edge $r{\rightarrow}y$ implies that changes to $r$ causally affect $y$.
To simulate this behavior, we create for each user an `intervention space' which includes a collection of possible interventions $r$ and their corresponding counterfactual outcomes $y$.
For our experiment, we simply take all pairs of reviews and ratings $(r,s)$
for a given beer to be the set of possible interventions and outcomes.
Textual reviews are featurized using a pre-trained BERT model \cite{devlin2018bert},
and numerical ratings $s\in[0,5]$ are transformed into binary choices $y=\{0,1\}$ by setting $y=1$ if the user's rating for that beer was above the median rating (for that beer), and $y=0$ otherwise.
Since learning requires observational data, for each user-beer pair $(u,x)$
we sample (in a way we describe shortly) one review-choice pair $(r,y)$ out of 100 unique reviews for that beer;
an example is presented in Figure \ref{fig:reviews_swap}.
This provides a sampled tuple $(u,x,r,y)$ expressing the behavior of a \emph{causal} user whose choices are affected by the review presented to her.
Together, $u,x$, and $r$ (as an embedding) include 866 features.


Finally, to model the effects of changing environments,
we consider an environment variable $\env$ that encodes the year,
expressing the idea that different years may express different `trends' in which beer \emph{types}\footnote{We create four beer `types' by aggregating beers of similar style. For example, 
the styles \emph{Doppelbock}, \emph{Dortmunder}, \emph{Dunkel}, \emph{Dunkelweizen}, and \emph{Dunkler} were all attributed to the same type.}
are more (and less) fashionable.
To implement this, we sample review-choice pairs for users within each year in a way that introduces a pre-determined amount of correlation between choices and beer types.
The chosen per-year correlation levels is plotted in Figure \ref{fig:corrs_beer}. Notice the drastic change in fashions in 2007 and 2011.



\paragraph{Training and testing.}
We train and evaluate one model per year. For each year $\env \in \{2006,\dots,2012\}$, training is performed on data from 
years $\{2002,\dots,\env-1\}$ and tested on $\env$.
In this way, fashions regarding beer type accumulate over time.

\paragraph{Models.}
We learn a linear model that takes as input the concatenation of $u,x,r$.
The learning objective includes a binary cross entropy loss, 
and marginal MMD as regularization \cite{gretton2012kernel}
(since we model users as causal; see \S \ref{sec:learning}).
We trained all models for 700 epochs with $lr=0.01$ and batches of size 1024,
and set $\lambda=100$.
Results are averaged over five runs with different random seeds.

\subsection{\emph{RecFashion} (anti-causal users)}

\paragraph{Original Dataset Statistics.}
The \textit{Fashion Product Images} dataset includes 
a large collection of fashion items, described by an image
and additional attributes such as:
season, gender, base color, usage, year, and product display name.
Items are organized by category, sub-category, and type;
we focus on the \emph{apparel} category.
Table \ref{tab:fashion_stats} provides summary statistics.


\begin{figure}[t!]
    \centering
    \includegraphics[width=0.8\linewidth]{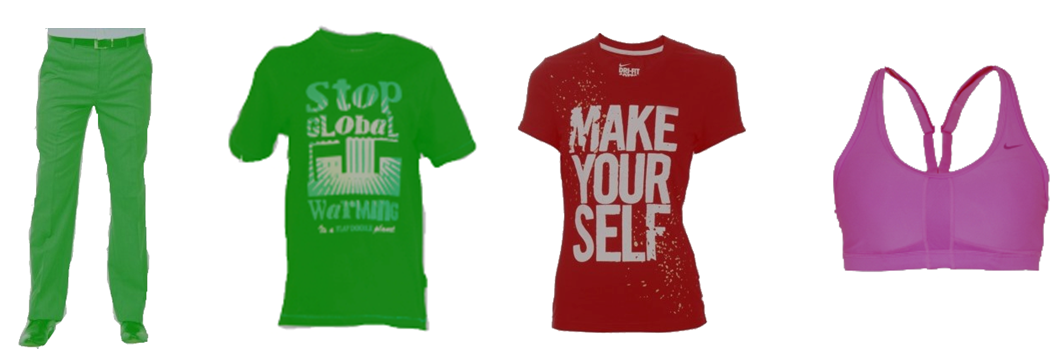}
    \caption{Fashion items in the \textit{RecFashion} dataset with recommended colors. On the left side are \textcolor{green}{green} recommendations and on the right side are \textcolor{red}{red} recommendations.}
    \label{fig:fashion_items}
\end{figure}

\begin{table}[h!]
    \centering
    \caption{Original \textit{Fashion Product Images} dataset statistics.}
    \begin{tabular}{l|l} \toprule 
        number of items & $44,447$ \\
        main categories & $7$ \\
        sub-categories & $45$ \\
        types & $142$ \\\bottomrule
    \end{tabular}
    \label{tab:fashion_stats}
\end{table}

\paragraph{Data Generation Process.}
The original dataset does not include user choices (or any other form of user behavior).
To simulate user choices, we imagine a setting were the platform recommends to each user an item by presenting an image of the item ($x$)
in a certain color ($r$).
We set $x$ to be the item's grayscale image,
and set $r$ to be a colorization of that image into one of two colors: red or green.
Users then choose whether to buy the item or not, $y \in \{0,1\}$.
We then model users as choosing primarily on the basis
of the `gender' attribute of items, $x_g \in \{0,1\}$,
and set $y=x_g$ w.p. 0.75 and $y=1-x_g$ otherwise.

Since users in this experiments are anti-causal,
they act under the belief that changes in $y$ affect $r$
(here we do not make use of the edge $y{\rightarrow}x$).
Note that $\env$ also affects $r$.
We implement this joint influence of $\env,y$ on $r$ by assigning colors to images in a way that obtains a certain level of correlation between the color $r \in \{\mathrm{red,green}\}$ and choices $y$.
Technically, we associate with each environment $\env$ a parameter $p_\env \in [0,1]$. Then, using a color variable $c=0$ for red and $c=1$ for green
we assign for each item its color as $c=y$ w.p. $p_e$, and $c=1-y$ otherwise.
Thus, different environments entail different conditional distributions
$P(r=\mathrm{red}|y=1)=P(r=\mathrm{green}|y=0)=p$,
which reflect an anti-causal structure.
Finally, given the sampled $c$,
we colorize the image $x$ as follows:
if $c=1$, we set $x_{R}\leftarrow0.5 + 0.2x_{R} $, $x_{G}\leftarrow 0.7x_{G} $, $x_{B}\leftarrow 0.7x_{B} $;
if $c=0$, we set $x_{G}\leftarrow0.5 + 0.2x_{G} $, $x_{R}\leftarrow 0.7x_{R} $, $x_{B}\leftarrow 0.7x_{B} $ ($R,G,B$ are the color channels).
Note that this means users do not observe $x,r$ independently, but
rather a colored image that is a product of both $x$ and $r$.

\paragraph{Training and testing.}
We run eight experiments that differ in the average degree of correlation in the training sets, for average correlation values of $p \in \{0.1, 0.2, \dots, 0.8\}$.
Each experimental condition ($p$) includes training data from six environments $\env$, with correlations
$p_env \in \{p - 0.025, p + 0.025,p - 0.05, p + 0.05, p - 0.1, p + 0.1\}$
(their average is $p$).


\paragraph{Models.}
For the model We used a feed forward neural network with three hidden layers and a hidden dimension of size 256, ReLU activation function and $NLL$ as our base loss function.
For computational efficiency, input images were resized to $14\times14$.
The learning objective includes a binary cross entropy loss, 
and a conditional DeepCORAL regularizer  \cite{sun2016deep}
(since we model users as anti-causal; see \S \ref{sec:learning}).
We set $\lambda=5000$ in the first 125 epochs and $\lambda=1$ in the rest,
and trained the model for 1,900 epochs with $lr=0.001$ and batches of size 1024.

\section{Loss Functions.}
\label{sec:app_learning}
We train all of our models with either the \textit{CORAL} or \textit{MMD} loss. Empirically, we found that \textit{CORAL} we more stable in the \textit{RecFashion} experiments and. In the \textit{RecBeer} experiments, models trained with the \textit{MMD} loss consistently outperformed those who were not. 
When conditioning on the label $y$, we compute $l_{dist}$ (either $l_{CORAL}$ or $l_{MMD}$) separately for cases where $y=1$ and $y=0$. We describe here both loss functions.

\paragraph{\textit{CORAL} Loss.}

The \textit{CORAL} loss is the distance between the second-order statistics of two feature representations, corresponding to different $z$:
\begin{equation*}
    l_{CORAL}(f(x,r),z) = \frac{1}{d^2} || C_{z} - C_{z'} ||^{2}_{F}
\end{equation*}

where $|| \cdot ||^2_{F}$ denotes the squared matrix Frobenius norm. The covariance matrices of the source and target data are given by:

\begin{align*}
    C_z =& \frac{1}{n_z - 1} (\phi(x(z),r)^\top \phi(X(z),r) \\
    &- \frac{1}{n_z}(\textbf{1}^\top \phi(x(z),r))^\top (\textbf{1}^\top \phi(x(z),r)))
\end{align*}

where \textbf{1} is a column vector with all elements equal to 1, and $\phi(\cdot)$ is the feature representation.

\paragraph{\textit{MMD}.}

Maximum mean discrepancy (\textit{MMD}) measures distances between mean embeddings of features. That is, when we have distributions $P$ and $Q$ over a set $\mathcal{X}$. The \textit{MMD} is defined by a feature map $\phi : \mathcal{X} \rightarrow \mathcal{H}$, where $\mathcal{H}$ is what's called a reproducing kernel Hilbert space. In general, the \textit{MMD} is
\begin{equation*}
    \text{MMD}(P,Q) = || \mathbb{E}_{X}[\phi(X)] - \mathbb{E}_{Y}[\phi(Y)] ||_{\mathcal{H}}
\end{equation*}

For use of the MMD loss for causal representation learning, see \citet{veitch2021counterfactual}.

\end{document}